\documentclass[pdflatex,sn-mathphys-num]{sn-jnl}

\usepackage{amsthm}
\usepackage{amssymb}
\usepackage[english]{babel}

\usepackage{amsmath}
\usepackage[title]{appendix}
\usepackage{amsfonts}
\usepackage{graphicx}
\usepackage{hyperref}

\usepackage{lineno}
\usepackage{soul}
\usepackage{ulem}

\newtheorem{definition}{Definition}
\newtheorem{theorem}{Theorem}
\newtheorem{remark}{Remark}
\newtheorem{corollary}{Corollary}
\newtheorem{lemma}{Lemma}

\begin{document}

\title{Stochastic Reservoir Computers}

\author*[1]{\fnm{Peter J.} \sur{Ehlers}}\email{ehlersp@arizona.edu}

\author[2]{\fnm{Hendra I.} \sur{Nurdin}}

\author*[1]{\fnm{Daniel} \sur{Soh}} \email{danielsoh@optics.arizona.edu}

\affil*[1]{\orgdiv{Wyant College of Optical Sciences}, \orgname{University of Arizona}, \orgaddress{\city{Tuscon}, \state{AZ}, \country{US}}}

\affil[2]{\orgdiv{School of Electrical Engineering and Telecommunications}, \orgname{University of New South Wales}, \orgaddress{\city{Sydney}, \country{Australia}}}

\abstract{Reservoir computing is a form of machine learning that utilizes nonlinear dynamical systems to perform complex tasks in a cost-effective manner when compared to typical neural networks. Many recent advancements in reservoir computing, in particular quantum reservoir computing, make use of reservoirs that are inherently stochastic. However, the theoretical justification for using these systems has not yet been well established. In this paper, we investigate the universality of stochastic reservoir computers, in which we use a stochastic system for reservoir computing using the probabilities of each reservoir state as the readout instead of the states themselves. In stochastic reservoir computing, the number of distinct states of the entire reservoir computer can potentially scale exponentially with the size of the reservoir hardware, offering the advantage of compact device size. We prove that classes of stochastic echo state networks, and therefore the class of all stochastic reservoir computers, are universal approximating classes. We also investigate the performance of two practical examples of stochastic reservoir computers in classification and chaotic time series prediction. While shot noise is a limiting factor in the performance of stochastic reservoir computing, we show significantly improved performance compared to a deterministic reservoir computer with similar hardware in cases where the effects of noise are small.}

\maketitle

Reservoir computing is a form of machine learning in which inputs are sequentially fed into a nonlinear dynamical system, whose state is measured and processed using an affine transformation that is trained to produce an output that closely matches some target sequence of interest. This process of feeding inputs into a system and extracting readouts is iterative, making reservoir computers (RCs) particularly well-suited to time-series modeling tasks \cite{antonik2018using, Chen:2022, LPHGBO17,larger2017high,canaday2018rapid,pathak2018model, Rafayelyan20}, and has also been successful for classification tasks \cite{Antonik:2020, Bianchi:2021, Carroll:2021,AKT21} among many possible uses \cite{Tanaka:2019}. The training performed on the readouts is a simple linear regression that can be solved exactly, giving RCs a relatively low computational cost in comparison to the neural networks in use today. The reason that simple linear regression suffices for complex time series modeling is that the nonlinear dynamical system provides readouts that are informationally rich in the inputs, realizing a nonlinear transformation of not only the current input but recent neighboring inputs as well. Using a naturally occurring physical system for the reservoir is a promising avenue toward fast and efficient computation of difficult tasks, as the natural dynamics of the system would perform the nonlinear transformations automatically, leaving only linear computations in the final output readout for a standard computer to perform.

The primary limitation of the reservoir computing approach is that a particular reservoir computing system is only suitable for a specific subset of problems. This limitation stems from the fact that the transformations applied by the reservoir are fixed. Consequently, a given reservoir will consistently process inputs in the same manner, which restricts the variety of potential outputs it can generate. Nevertheless, classes of RCs have been shown to be universal approximating classes, such as state-affine systems \cite{Grigoryeva:2018a} and echo state networks (ESNs) \cite{Grigoryeva:2018b}, among others \cite{Bishop:2022, Chen:2019, Gonon:2020, Hart:2020, LiBoyu:2023, LiZhen:2023, Martinez:2023, Nokkala:2021, Sugiura:2024}.  However, achieving an exactly accurate reservoir computing system may be impractical as it may require a large reservoir along with an extreme level of numerical precision. Consequently, one often must balance approximation accuracy with the practical limitations on size.

The principal finding of this paper is that classes of stochastic echo state networks — a type of stochastic RC — constitute universal approximating classes. Universality is an important property for classes of RCs as it guarantees that any task we want to solve can be approximated to arbitrary precision using the architecture that defines the class. Without the universality property, choosing a specific reservoir architecture in the hope of approximating a certain target function could result in complete failure. We provide a rigorously detailed mathematical proof that demonstrates universality in a promising scalable architecture — the stochastic reservoir computer. This is a critical milestone essential for nurturing and investing in an important scalable class of RCs. As outlined in the next section, our definition of a stochastic RC deviates from the conventional reservoir computing model by utilizing the probabilities of each observed reservoir state for readouts, rather than the states themselves. This modification necessitates our detailed analysis of these probabilities to confirm that they are complex, nonlinear functions of the inputs, requiring only an affine transformation for universal approximation.

The use of stochastic RCs holds significant potential because the reservoir can be designed such that the number of computational nodes and, therefore, the possible outcomes scale exponentially with the hardware size. For example, in a qubit-based quantum computer, one can perform a unitary rotation of a set of $n_q$ qubits based on the input at a given time step, then measure the qubits in the computational basis. By running this process many times, we can obtain an estimate of the probability of measuring each of the $2^{n_q}$ distinct bit strings and use them as readouts for training to get an output. Another potentially promising implementation of stochastic RCs is on recently proposed classical (non-quantum) probabilistic computers that are based on probabilistic bits (p-bits) and other classical stochastic devices \cite{chowdhury24,roques-carmes23,misra22}, including classical devices that exploit quantum effects \cite{roques-carmes23}. The requirement that the stochastic RC must be run many times to get an estimate of the probabilities is a potential setback, representing a tradeoff between more computational nodes and machine runtime. However, this approach becomes valuable in scenarios where a deterministic computer lacks sufficient computational nodes to perform the necessary calculations. In such cases, a stochastic machine can achieve an equivalent number of computational nodes using a much smaller device. We examine two different designs for a stochastic RC more closely in the Results section, where we define the exponential scaling more concretely.

\subsection*{Stochastic Reservoir Computers}
\label{sec:SRC}
In this section, we will give our precise definition of a stochastic RC. There are two primary differences between this class of RCs and the deterministic RCs that are well established in the literature. First, the stochastic reservoir gives one of $M$ possible outcomes in $\mathcal{X} = \{x^{(1)},\dots,x^{(M)}\}$ at each time step $k$. Second, instead of using the outcomes directly as readouts for training to acquire an output, we use the probabilities $P_{k,a}$ of observing a distinct outcome $x^{(a)}\in\mathcal{X}$ labeled by an index $a\in\{1,\dots,M\}$ at each time step $k$ as readouts instead. Since in a deterministic RC the measured state of the system $x_k$ also serves as the readout used for training, we need to introduce language to distinguish these two concepts. For stochastic RCs, we will refer to the measured state $x_k$ as the \textit{outcome} of the reservoir, while the probabilities $P_{k,a}$ of observing an outcome $x^{(a)}$ are called the \textit{readouts}, as these are what we ultimately use for training.

The reason this construction is useful is because it is easy to achieve a large number of outcomes $M$ by stringing together a small number of stochastic systems with only a few outcomes individually and reading the outcomes out as a vector $x_k$ at each time step. For example, a binary stochastic process (a p-bit) has only two possible outcomes, but $L$ of these together have $M=2^L$ distinct outcomes, and in general a vector outcome $x$ whose components each have $m$ possible outcomes will have as an ensemble $M=m^L$ distinct outcomes. This makes it easy to generate a RC with a large number of computational nodes using a fairly small device, as we will illustrate in the Results section.

Mathematically, the action of the stochastic reservoir is described by
\begin{align}
\label{eq:RCreadout}
    x_{k+1} &= F(x_k, u_k) \\
    F(x, u) &\in \mathcal{X} = \{x^{(1)},\dots,x^{(M)}\},
\end{align}
Thus the outcome of a reservoir measurement at time step $k$ is restricted to a specific set of $M$ possible vectors of length $L$, each with its own probability of occurring. The probabilities involved in this process are defined as
\begin{align}
    P_{k,a} &\equiv P(x_k=x^{(a)})  \\
    p_{ab}(u) &\equiv P(F(x, u)=x^{(a)}|x=x^{(b)}).
\label{eq:pabdef}
\end{align}
In other words, $P_{k,a}$ is the probability at time step $k$ of measuring outcome $x^{(a)}$, while $p_{ab}(u)$ is the conditional probability of measuring $F(x, u)=x^{(a)}$ with an input $u$ given that the starting outcome was $x^{(b)}$. Note that the conditional probability has no dependence on $k$ as it describes the effect that the reservoir has on the probabilities of each outcome, and the reservoir dynamics do not change over time.

Using these probabilities, we can then write a recursive definition of the probabilities $P_{k,a}$ using the definition of conditional probability to get
\begin{align}
    P_{k+1,a} &= P(F(x_k, u_k) = x^{(a)}) \\
    &= \sum_{b=1}^{M}P(F(x_k, u_k) = x^{(a)} | x_k = x^{(b)}) P(x_k = x^{(b)}) \\
    &= \sum_{b=1}^{M}p_{ab}(u) P_{k,b}
\end{align}
This can be written more succinctly in matrix form as $P_{k+1} = p(u_k) P_k$. Therefore, any stochastic RC as defined here is a Markov chain \cite{Ross:2023}, in particular a controlled discrete-time Markov chain \cite{Meyn:2009} in which the input sequence $(u_k)_{k\in\mathbb{Z}}=\{u_k\left|k\in\mathbb{Z}\right.\}$ acts as the fixed control sequence for the controlled Markov chain. The controlled transition probability matrix $p(u)$ whose elements are defined by \eqref{eq:pabdef} is determined by the dynamical reservoir. This means that, rather than using the finite number of outcomes in the set $\mathcal{X}$ as readouts, we can instead use the probability vectors $P_k$ as our readouts and train those to fit a target series. Thus the controlled Markov chain forms the basis of stochastic reservoir computing, and we can then write the RC equations as 
\begin{align}
\label{eq:probres}
    P_{k+1} &= p(u_k) P_k \\
    \hat{y}_k &= W^\top P_k,
\label{eq:yhat}
\end{align}
where $\hat{y}_k$ is our estimate for some target value $y_k$ which is fit by training the $M$-dimensional weight vector $W$. This causes $P_k$ and $\hat{y}_k$ to have dependence on all past inputs $\{u_{k-1}, u_{k-2}, \dots\}$ while remaining independent of the current input $u_k$. The reasoning for this is that physical systems do not have an instantaneous response to their environment, so in modeling the functional dependence of such a system with a RC we should also avoid using the current value of the input.

In practice, a large stochastic RC that would be suitable for finding good approximations to difficult problems will have an enormous state space that will make an analytic calculation of the probabilities $P_k$ virtually impossible. Instead, we can obtain estimates for $P_k$ by performing many independent runs of the stochastic RC through the training and test data and averaging over the results. In the Results section, we will examine the effects of estimating $P_{k,a}$ using a finite number of measurements versus using the exact probabilities.

\subsection*{Stochastic Echo State Networks}
Our proof of universality will focus on a specific type of stochastic reservoir computing based off of the echo state network (ESN). The action of the stochastic ESN is described by
\begin{align}
\label{eq:stocESN}
    z_k &= A x_k +B u_k \\
    x_{k+1,i} &= \Gamma(z_{k,i}),
\end{align}
where $\Gamma(\zeta)\in\{\xi^{(1)},\dots,\xi^{(m)}\}$ is a scalar stochastic activation function with $m$ outcomes that is a function of some scalar input $\zeta$. The probabilities then evolve via a controlled Markov process according to Eq. \eqref{eq:probres}, and the outputs are derived from the probabilities as in Eq. \eqref{eq:yhat}. The elements of the controlled transition matrix associated with the ESN are defined by
\begin{align}
    p_a(z) &= \prod_{i=1}^L\varrho_{a_i}(z_i)\\
    p_{ab}(u) &= p_a(Ax^{(b)}+Bu),
\label{eq:ESNprob}
\end{align}
where $\varrho_{a_i}(z_i) = P(\Gamma(z_i)=\xi^{(a_i)})$ is the probability of measuring $\xi^{(a_i)} \equiv x^{(a)}_i$ given the $i$th component of the vector $z$, and $L$ is the dimension of each $x_k$ in the ESN. Each $a_i$ is an integer between 1 and $m$ such that there is some bijective mapping between the set $\{a_i\left|i\in\{1,\dots,L\}\right.\}$ and the integers $a\in\{1,\dots,M=m^L\}$.  $\varrho(\zeta)$ is to be interpreted as the probability distribution associated with the scalar stochastic function $\Gamma(\zeta)$. We now have a controlled composite probability distribution $p_a(z)$ that determines the probability of measuring outcome $a\in\{1,\dots,M=m^L\}$ given a vector-valued input $z$, which decomposes into a product of independent controlled probabilities $\varrho_{a_i}(z_i)$ associated with each component of both $z$ and the measurement outcome $x^{(a)}$. Note that since we always work with bounded inputs and because there are finitely many possible values of $x$, every component of $z=Ax +Bu$ will also be bounded for any given $A$ and $B$. 

\section*{Results} 
\label{sec:examples}

\subsection*{Universality of Stochastic ESNs}

Our main result is a rigorous proof of the universality of stochastic echo state networks. More precisely, we have proven that a class of stochastic echo state networks that share a common activation function form a universal approximating class, provided that the stochastic behavior of the activation function has certain properties. This universality proof guarantees that, given a specific activation function that meets the criteria for universality, there will be a stochastic ESN that can be used to approximate any given task to arbitrary precision. In practical terms, this means that if an efficient design for implementing these stochastic ESNs were found and could be scaled up relatively easily, then that design would be sufficient for building an ESN that will successfully approximate any problem of interest. Two possible candidates for such designs are presented and analysed in the following sections.

The formal statement of our theorem establishing the universality of stochastic ESNs is as follows:
\begin{theorem}[Universality]
\label{thm:universal}
    With the set of uniformly bounded sequences $K_{R_u} \subset (\mathbb{R}^n)^{\mathbb{Z}_-}$ and a weighted metric $||\cdot||_w$ defined in Thm. \ref{thm:UCFM}, let $\mathcal{G}_w^{\varrho}$ be the class of functionals generated by stochastic ESNs defined by a controlled probability distribution $\varrho(\zeta)$ that is defined on the interval $\zeta\in[-R_\zeta,R_\zeta]$ for some $R_\zeta>0$, is continuous in $\zeta$, and satisfies $\varrho^\top(\zeta_1)\varrho(\zeta_2)>0$ for all $\zeta_1, \zeta_2\in[-R_\zeta, R_\zeta]$. If, for the invertible subspace $\mathcal{V}_{R_u}^p$ with respect to the controlled transition matrix $p(u)$ corresponding to this class of ESNs, there exists a vector $V\in\mathcal{V}_{R_u}^p$ for which $V^\top\varrho(\zeta)$ is strictly monotonic on the interval $[-R_{\mathrm{mono}}, R_{\mathrm{mono}}]$ for some $R_{\mathrm{mono}}>0$, then the class $\mathcal{G}_w^{\varrho}$ is dense in the space of functionals over the compact metric space $(K_{R_u}, ||\cdot||_w)$.
\end{theorem}
A detailed proof of this theorem and a definition of an invertible subspace with respect to a controlled transition matrix $p(u)$ is given in the Methods section. We also give as a corollary to Theorem \ref{thm:universal} the universality statement for the specific case of $m=2$, which has fewer conditions and is relevant to both the numerical examples in this work as well as for broader applications such as in quantum computers, as well as another corollary for the universality of a broader class of stochastic RCs.

\subsection*{Numerical Examples}

In the next couple of sections we will demonstrate the convergence and performance of stochastic ESNs using a couple of examples. In both examples, we use a stochastic activation function $\Gamma(\zeta)$ that outputs either $0$ or $1$ with controlled probabilities $\varrho_0(\zeta)$ and $\varrho_1(\zeta) = 1 - \varrho_0(\zeta)$ that have the required properties which ensures that the class $\mathcal{G}_w^{\varrho}$ of stochastic ESNs using $\varrho(\zeta)$ is a universal approximating class. A precise statement of these properties is given at the top of the Methods section. The designs of the two examples are also meant to emphasize the practicality of implementing a stochastic RC using readily available hardware. In particular, they can be scaled up in a relatively easy way to produce an exponentially scaling number of computational nodes, with the drawback of potentially needing a larger number of measurements to estimate the probabilities to the required level of precision.

To test the performance of these stochastic ESNs, we will use two different tasks, the first of which is the Sine-Square wave identification task \cite{Dudas:2023}. This task requires the ESN to identify whether a given section of the input is a sine wave or a square wave, yielding a output of $\hat{y}_k<0.5$ for a sine wave and $\hat{y}_k>0.5$ for a square wave. The other task we will use in this work is the Lorenz $X$ task. The Lorenz $X$ task requires the ESN to approximate a chaotic dynamical system and predict what the sequence will be 1 time step into the future. In other words, using the input sequence $(u_k)_{k\in\mathbb{Z}} = (y_{k-1})_{k\in\mathbb{Z}}$, we want the ESN to successfully predict $(y_k)_{k\in\mathbb{Z}}$ in this task. In both tasks, we run the ESN for roughly 100 steps so that the system will forget its initial state, then train the ESN using 3000 steps, then test the performance on roughly 100 time steps afterward. Further details about these tasks and the simulation parameters are given in Supplementary Information \ref{sec:apnddetails}. The data and code used to obtain our results can be found in a figshare repository \cite{figshare}.

\subsubsection*{Qubit Reservoir Network}

\begin{figure}
    \centering\includegraphics[width=0.95\linewidth]{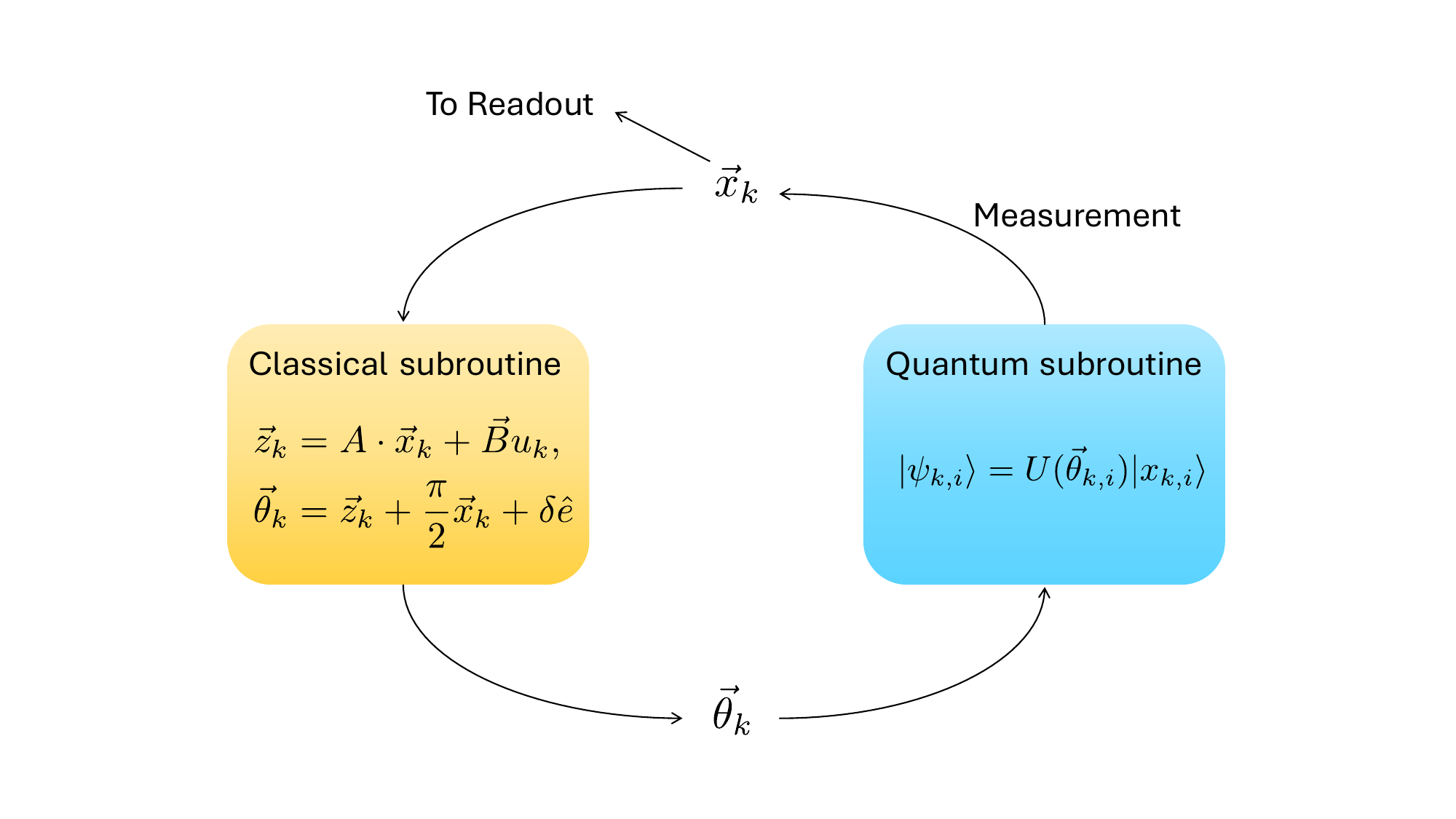}
    \caption{Diagram of the design for the qubit reservoir network.} 
    \label{fig:QRN}
\end{figure}

The qubit reservoir network is an example of a stochastic echo state network that utilizes simple, single-qubit gate operations applied to an array of qubits to generate a nonlinear stochastic activation function. This is a new hybrid quantum-classical approach that capitalizes on the ease of performing linear operators on classical computers while also leveraging a quantum device for rapid nonlinear transformations, similar to a variational quantum algorithm. The routine for each time step of the ESN is illustrated in Fig. \ref{fig:QRN}, where the linear transformation is applied classically, and the $z$ values are passed onto each qubit using a transformation operator of the form $U(\theta)=e^{-i\sigma_x \theta}$. First, at each time step $k$ we use the most recent array of qubit measurements $x_k$ and the current input $u_k$ to calculate $z_k = A x_k+B u_k$ on a classical computer. Then, we apply the transformation operator $U(z_{k,i}+\frac{\pi}{2}x_{k,i}+\delta)$ to each qubit, where $\delta$ is some constant shift parameter. The extra factor of $\frac{\pi}{2}x_{k,i}$ resets the qubits to $|0\rangle$ from the states $|x_{k,i}\rangle$ they collapsed to after the previous measurement. Finally, each qubit is then measured in the computational basis, which is recorded as $x_{k+1}$ and used immediately in the next time step. We then repeat this cycle many times and record the number of times each distinct measurement outcome $x^{(a)}$ appears at each time step $k$ to obtain estimates of their probabilities $P_{k,a}$ of occurring. These probabilities are then used as the readout vector for training and obtaining estimates $\hat{y}_k = W^\top P_k$ of the target output $y_k$. The initial state of the qubits has a negligible effect on the results because the ESN necessarily has the uniform convergence and fading memory properties, which are defined in the Methods section. These properties make it so that the initial condition will be forgotten as the reservoir is driven to some specific evolving configuration by the inputs. By letting the system run for a number of time steps before training, the dependence on the initial condition is effectively removed.

For this class of stochastic ESNs, the scalar activation function $\Gamma(\zeta)$ is the measurement of a single qubit in the computational basis after applying $U(z_{k,i}+\frac{\pi}{2}x_{k,i}+\delta)$. The conditional probabilities associated with $\Gamma$ for a single node ESN are given by
\begin{align}
    P(\Gamma = \xi'|\xi,u) &= |\langle \xi'|U(A \xi +B u+\frac{\pi}{2}\xi+\delta)|\xi\rangle|^2 \\
    &= |\langle \xi'|U(A \xi +B u+\delta)e^{-i\frac{\pi}{2}\sigma_x\xi}|\xi\rangle|^2 \\
    &= |\langle \xi'|U(A x +B u+\delta)(-i\sigma_x)^\xi|\xi\rangle|^2 \\
    &= |\langle \xi'|U(A \xi +B u+\delta)|0\rangle|^2,
\end{align}
where in these equations $A$ and $B$ are scalars and we have used that $\xi=\{0,1\}$ to simplify $(\sigma_x)^0|0\rangle=(\sigma_x)^1|1\rangle=|0\rangle$. This has the form $P(\Gamma = \xi^{(a)}|\xi,u) = \varrho_a(\zeta)$ with $\zeta=A \xi + B u$, so this is a stochastic echo state network with the controlled probability distribution
\begin{align}
    \varrho_a(\zeta) = |\langle a|U(\zeta+\delta)|0\rangle|^2
\end{align}
with $a=0,1$. Since $\varrho_1(\zeta) = 1-\varrho_0(\zeta)$ for any $\zeta$, the system is fully specified by the controlled scalar probability
\begin{align}
\label{eq:QRNprob}
    \varrho_1(\zeta) = |\langle 1|U(\zeta+\delta)|0\rangle|^2 = |\langle 1|e^{-i(\zeta+\delta)\sigma_x}|0\rangle|^2 = \sin^2(\zeta+\delta).
\end{align}
This class of stochastic ESNs can be shown to be universal when $\delta=\frac{\pi}{4}$ and when the matrix $A$ and vector $B$ are chosen so that all elements of $z_k = A x_k + B u_k$ satisfy $|z_{k,i}|<\frac{\pi}{4}$ at every time step $k$. We show this in Supplementary Information \ref{sec:apnduniversal}. The value of $\delta=\frac{\pi}{4}$ is chosen because it allows for the largest possible upper bound for $|z_{k,i}|$ that still leaves $\varrho_1(\zeta)$ monotonic and therefore generates a universal approximating class. Other choices can still achieve universality, but will require tighter bounds on $|z_{k,i}|$ and thus better numerical precision. One could use an ESN that goes outside of the bound on $|z_{k,i}|$, and it may even perform well, but there is no guarantee that a useful ESN exists outside this window.

\begin{figure}
    \centering
    \includegraphics[width=0.45\linewidth]{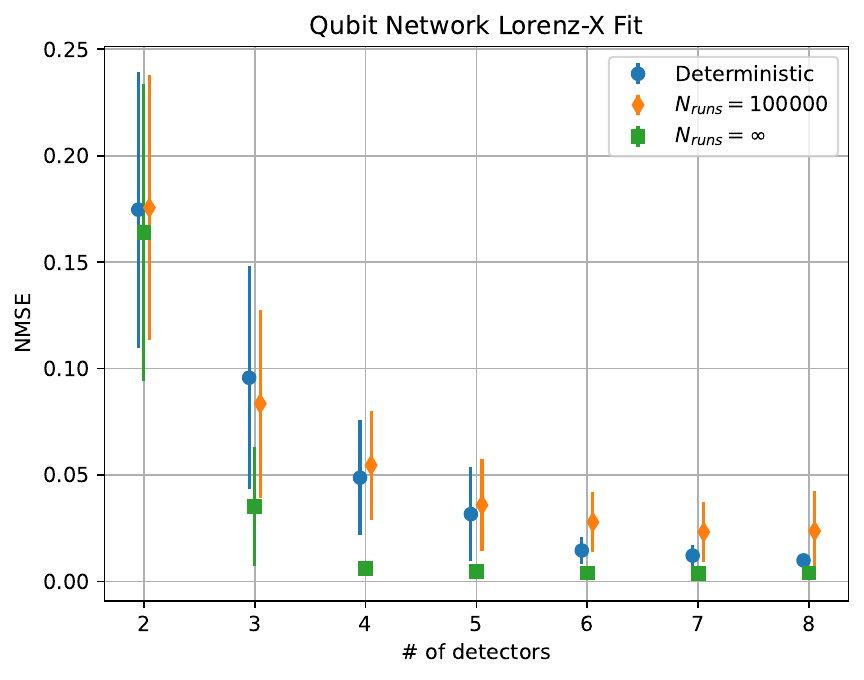}
    \includegraphics[width=0.45\linewidth]{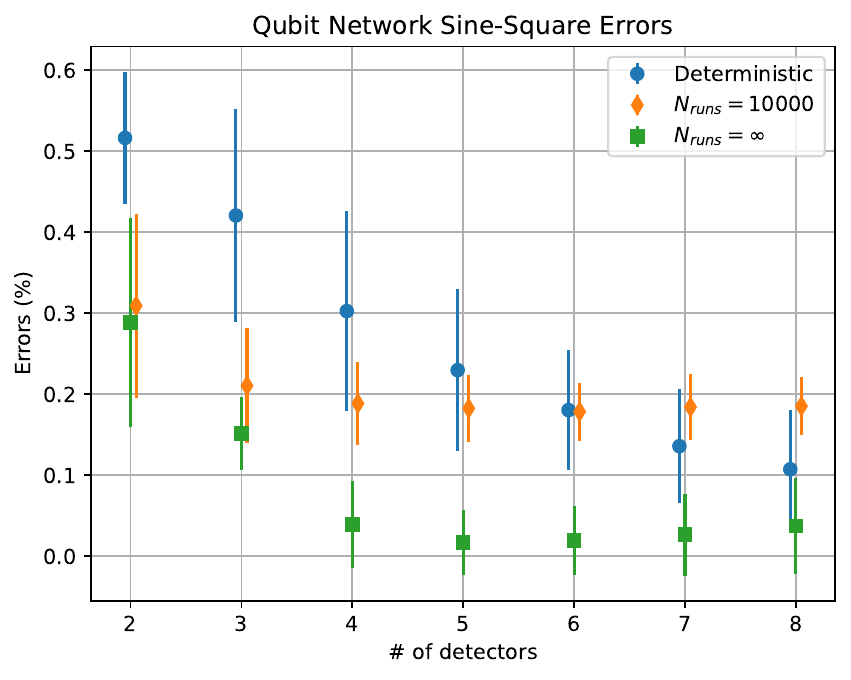}
    \caption{The left plot shows average NMSE values for the Lorenz $X$ task, while the right plot shows the average error percentage for the Sine-Square wave identification task, both as a function of the number of detectors using the qubit reservoir network. The blue dots correspond to the deterministic ESN, the orange dots correspond to the stochastic ESN averaged over 100000 runs for the Lorenz $X$ task and 10000 runs for the Sine-Square wave task, and the green dots correspond to the stochastic ESN results in the theoretical limit of infinite statistics, using the exact probabilities. Each data point represents the average over 100 random choices of $A$ and $B$ for the Lorenz $X$ task and 1000 choices for the Sine-Square wave task, and the error bars give the standard deviations of each data point over the samples of ESNs with different choices of $A$ and $B$.}
    \label{fig:SineSquareError}
\end{figure}

The left plot of Fig. \ref{fig:SineSquareError} shows a comparison of NMSE values for the Lorenz $X$ task averaged over 100 different choices of $A$ matrices and $B$ vectors for the ESN, while the right plot shows a comparison of the percentage of errors recorded for the Sine-Square wave identification task averaged over 1000 choices. In this context, ``\# of detectors" refers to the number of physical qubits that make up the reservoir. We see in the figure that the stochastic ESN using the exact probabilities performs much better than the deterministic ESN using roughly the same hardware. With finite statistics, the stochastic ESN still performs well for a small number of detectors, but beyond 4 detectors the values of the error measure seems to level off, and the deterministic ESN begins to outperform it on average by 4 detectors in the Lorenz $X$ task and at 7 detectors in the Sine-Square wave task. 

An explanation for this can be found by analysing how the linear regression behaves under the noise introduced by the stochastic reservoir, which is detailed in Supplementary Information \ref{sec:apndnoise}. The takeaway is that if $N_{\mathrm{runs}}$ is not large enough to suppress the shot noise and resolve all of the basis vectors formed from the $P_k$'s, then the performance of the stochastic reservoir will be reduced compared to what it could be with the exact probabilities.

\subsubsection*{Stochastic Optical Network}

\begin{figure}
    \centering
    \includegraphics[width=0.95\linewidth]{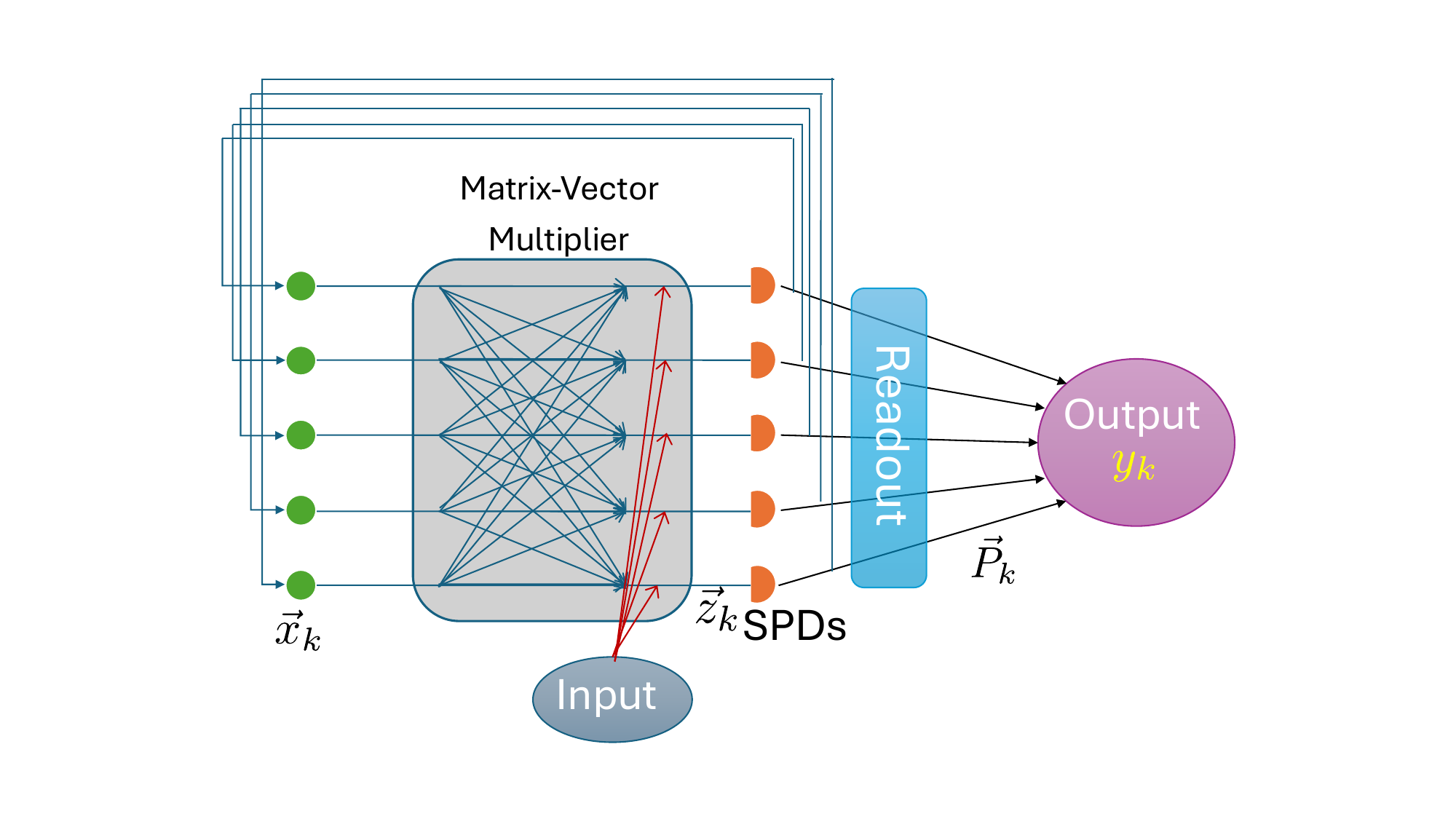}
    \caption{Diagram of the design for the stochastic optical network.} 
    \label{fig:SON}
\end{figure}

The second example of a stochastic echo state network is an optical setup based off of the optical neural network proposed in Ref. \cite{Ma:2023pre}. The novel approach in their design is the use of single photon detectors (SPDs) in the hidden layer of their neural network. We use this concept in the design of a stochastic ESN, where very low intensity lasers have their photon number measured using SPDs, generating a stochastic activation function. In this work, we focus exclusively on results using coherent laser light, but incoherent beams may be used as well, as in the main result of Ref. \cite{Ma:2023pre}. In contrast to the previous example, this design is almost entirely classical aside from the quantized measurement of light and can be built today only using well-established hardware that has existed for decades. This demonstrates that while stochastic RCs have a deep connection to quantum computing due to the stochastic nature of quantum measurement, they are not inherently quantum devices.

The design of the device is shown in Fig. \ref{fig:SON}. Here, an array of very low intensity coherent lasers which are either turned on or off at each time step $k$ encodes the vector $x_k$, where $x_{k,i}\in\{0,1\}$ for every $k$ and $i$. This laser array is passed through a matrix-vector multiplier and joined with another laser array whose intensities are proportional to the input $u_k$. This applies the linear part of the reservoir action to produce $z_k = A x_k + B u_k$. Then, each laser in the array is measured with a single photon detector, which enacts the stochastic activation function $x_{k+1,i} = \Gamma(z_{k,i})$. In this setup, the SPD only measures whether or not any light is detected. For coherent light, the probability that the vacuum state is measured is given by $P(\Gamma=0|\alpha) = e^{-\alpha^*\alpha}$ for the coherent state parameter $\alpha$, and therefore $P(\Gamma=1|\alpha) = 1-e^{-\alpha^*\alpha}$. If we include a constant offset $d$ along with the input beams, then the controlled probability distribution associated with $\Gamma$ will be
\begin{align}
\label{eq:SONprob}
    \varrho_1(\zeta) = 1-e^{-(\zeta+d)^2}
\end{align}
This class of stochastic ESNs can be shown to be universal with $\delta=1$ and when the matrix $A$ and vector $B$ are chosen so that all elements of $z_k = A x_k + B u_k$ satisfy $|z_{k,i}|<1$ at every time step $k$.  We show this in Supplementary Information \ref{sec:apnduniversal}.

\begin{figure}
    \centering
    \includegraphics[width=0.45\linewidth]{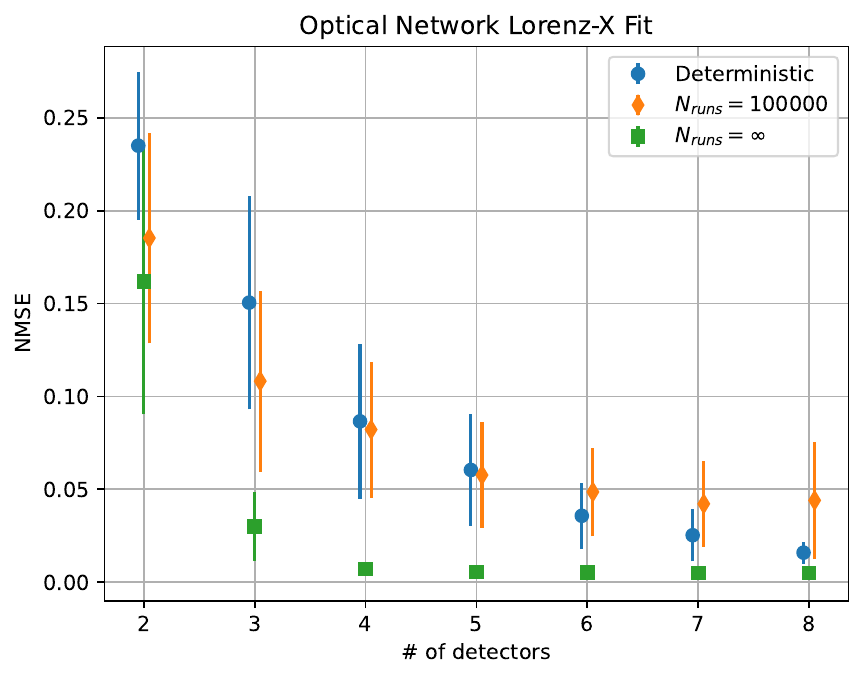}
    \includegraphics[width=0.45\linewidth]{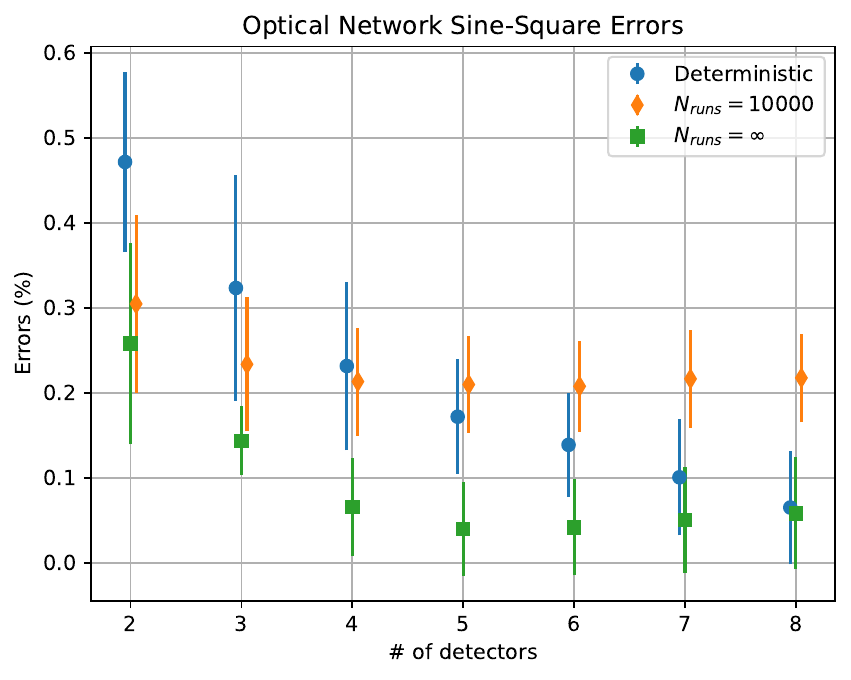}
    \caption{The left plot shows average NMSE values for the Lorenz $X$ task, while the right plot shows the average error percentage for the Sine-Square wave identification task, both as a function of the number of detectors using the stochastic optical network. The blue dots correspond to the deterministic ESN, the orange dots correspond to the stochastic ESN averaged over 100000 runs for the Lorenz $X$ task and 10000 runs for the Sine-Square wave task, and the green dots correspond to the stochastic ESN results in the theoretical limit of infinite statistics, using the exact probabilities. Each data point represents the average over 100 random choices of $A$ and $B$ for the Lorenz $X$ task and 1000 choices for the Sine-Square wave task, and the error bars give the standard deviations of each data point over the samples of ESNs with different choices of $A$ and $B$.}
    \label{fig:LorenzNMSE}
\end{figure}

The left plot of Fig. \ref{fig:LorenzNMSE} shows a comparison of the NMSE values associated with the Lorenz $X$ task, averaged over 100 different choices of $A$ matrices and $B$ vectors for the ESN, while the right plot shows a comparison of the percentage of errors recorded for the Sine-Square wave identification task averaged over 1000 choices. We see in the figure that the stochastic ESN using the exact probabilities once again performs much better than the deterministic one using similar hardware. With finite statistics, we see the same initially strong performance as in the previous example followed by a leveling off of improvement, with the deterministic ESN outperforming it on average by 6 detectors in the Lorenz $X$ task and by 5 detectors in the Sine-Square wave task.

\section*{Discussion}
\label{sec:conclusion}

In this work, we define stochastic RCs as RCs that produce stochastic outcomes, where the probability of each outcome rather than the individual instantiation of the outcome serves as the readout. We provided necessary and sufficient conditions for the contraction of the conditional probabilities, which directly lead to critical elements for class universality, specifically, uniform convergence and fading memory in the corresponding stochastic RC. We rigorously proved that stochastic echo state networks whose activation functions obey certain criteria form a universal approximating class. Then, we considered two example stochastic RC hardware platforms theoretically: a qubit reservoir network and a stochastic optical network. For both hardware platforms, we investigated two practical numerical task examples, namely the Sine-Square wave identification task and the Lorenz $X$ prediction task. We found in those examples that while the stochastic ESN using the exact probabilities performed much better than a deterministic ESN using similar hardware, having to get an estimate for the probabilities using a finite number of runs limited the performance of these stochastic ESNs.

From the plots of the average error measures on both the Lorenz $X$ and the Sine-Square wave tasks in Figs. \ref{fig:SineSquareError} and \ref{fig:LorenzNMSE}, we directly compared the performance of both of our examples. While the stochastic ESNs demonstrate superior performance over the deterministic ESNs with the same number of detectors, the finite number of trials for obtaining the probability distributions limited the performance of the stochastic ESNs. Moreover, we saw that generally the qubit reservoir network stochastic ESN was able to achieve equal to slightly better fits in both tasks in the stochastic case than the optical network stochastic ESN. On the other hand, the deterministic ESN seems to perform better at the Sine-Square wave tasks using the optical network activation function as opposed to the qubit reservoir activation function. Interestingly, this causes the stochastic optical network to show a bit of a better comparison to the deterministic case on the Lorenz $X$ task, while the qubit reservoir network has a more favorable comparison on the Sine-Square wave identification task. Since nearly everything was the same in this comparison between the qubit network ESN and the optical network ESN except for the activation function used, it suggests that different choices of activation function lead to changes in performance that are task dependent.

In addition to the direct use of stochastic RCs to perform specific computational tasks, these machines may also prove useful in exploring a conceptual question in quantum machine learning. Both the exponential scaling of the Hilbert space and entanglement have been claimed as potential sources of quantum advantage \cite{Gotting:2023, Kenigsberg:2006, Kora:2024, Mujal:2023}, and yet it is still unclear to what extent these properties can provide any form of computational speedup compared to a classical counterpart. Stochastic RCs as described here exhibit an exponential scaling of the state space, but do not in general employ any quantum entanglement, as can be seen with the examples in the Results section. Comparing the performance of these stochastic RCs with a genuine quantum RC that incorporates entanglement into its dynamics could clarify the advantages of quantum entanglement. By controlling for the scaling of the Hilbert space, this comparison would provide a clearer understanding of which aspects of quantum mechanics are beneficial for computational purposes.

Looking more closely at Figs. \ref{fig:SineSquareError} and \ref{fig:LorenzNMSE}, it appears that the error measures in both cases for the stochastic ESNs using the exact probabilities also level off after a certain point, much like the 10000 run ESNs as discussed in Supplementary Information \ref{sec:apndnoise}. In fact, if we keep increasing the number of detectors of the deterministic ESNs beyond the values shown in the figures, we would see the same leveling off of performance as well. Considering that adding more nodes will introduce basis vectors that always seem to get more difficult to resolve, it seems as though machine precision errors may substantially limit the performance of larger RCs. For further discussion on this point, see the end of section \ref{sec:apndnoise} in the Supplementary Information. One way to potentially avoid this would be to use an alternative to optimizing via matrix inversion, such as gradient descent. However, gradient descent methods have their own issues resolving small values \cite{McClean:2018, Kolen:2001, ShalevShwartz:2017} which would likely be present in this context. Another option would be to find reservoirs with basis vectors that are easy to resolve. These RCs would be able to achieve a manageable signal-to-noise ratio for all of its eigentasks (see Ref. \cite{Hu:2023}), but this would require a much stronger theoretical handle on how reservoir dynamics lead to strong fits, which is at present a topic in need of extensive research.

\section*{Methods}
Our proof of universality for stochastic RCs will follow the strategy used in Refs. \cite{Grigoryeva:2018a, Chen:2019}, proving that the class of stochastic RCs satisfies the criterion of the Stone-Weierstrass Theorem \cite[Theorem 7.3.1]{Dieudonne:2011}. In the next few sections, we will explicitly prove the universality of certain classes of stochastic ESNs, which implicitly proves the universality of the class containing all stochastic RCs with the contraction property because ESNs are a subclass of RCs.

To use the Stone-Weierstrass theorem, we first need to provide conditions on the controlled transition matrix $p(u)$ that guarantees that the stochastic RC will have the uniform convergence and fading memory properties, which we will describe in further detail in the next section. Then, we must establish that our class of stochastic ESNs obeys a polynomial algebra on a compact metric space and show that it separates points in that algebra. These properties are what the next two sections will demonstrate. Finally, we can then provide the proof of Theorem \ref{thm:universal} that utilizes all of these properties in conjunction with the Stone-Weierstrass theorem, and conclude with some useful corollaries.

The conditions that guarantee the universality of the associated approximating class are given here as reference. The reasons for asserting these conditions will be elaborated on in the proofs below. In terms of the probability distribution vector $\varrho(\zeta)$ defined on some interval $\zeta\in[-R_\zeta, R_\zeta]$ with $R_\zeta>0$, these conditions can be written as follows:
\begin{enumerate}
    \item $\varrho(\zeta)$ is continuous in $\zeta$ for all $\zeta\in[-R_\zeta, R_\zeta]$.
    
    \item $\varrho(\zeta)$ satisfies $\varrho^\top(\zeta_1)\varrho(\zeta_2)>0$ for all $\zeta_1, \zeta_2\in[-R_\zeta, R_\zeta]$.
    
    \item For the invertible subspace $\mathcal{V}_{R_u}^p$ defined in the Methods section, there exists a vector $V\in\mathcal{V}_{R_u}^p$ for which $V^\top\varrho(\zeta)$ is strictly monotonic on the interval $[-R_{\mathrm{mono}}, R_{\mathrm{mono}}]$ for some $R_{\mathrm{mono}}>0$.
\end{enumerate}
The first two criteria ensure that the reservoir computer has the uniform convergence and fading memory properties. More specifically, criterion 2 guarantees that the action of the reservoir is contracting in the probabilities over time, which in turn guarantees uniform convergence. This combined with the continuity of $\varrho(\zeta)$ also leads to the fading memory property. Criterion 3 is required in order to separate points in the space of input histories, which is required for a universal approximating class. The monotonicity of the vector component $V^\top\varrho(\zeta)$ ensures that any difference between inputs can be resolved, no matter how small. The requirement that the vector $V$ belongs to the invertible subspace is necessary to make sure that some vector $W$ exists that can pick out the monotonic component of $\varrho(\zeta)$ at any point in the past, which is necessary to prove separability of two input sequences that are identical after some time step in the past.

In the special case where $\varrho(\zeta) = \left(\begin{matrix} \varrho_1(\zeta) & (1-\varrho_1(\zeta)) \end{matrix}\right)^\top$, which applies to both numerical examples described earlier, the above conditions can be rewritten in terms of the scalar function $\varrho_1(\zeta)$, and will reduce to the following:

\begin{enumerate}
    \item $\varrho_1(\zeta)$ is continuous in $\zeta$ for all $\zeta\in[-R_\zeta, R_\zeta]$.
    
    \item either $\varrho_1(\zeta)\neq 0$ or $\varrho_1(\zeta)\neq 1$ for all $\zeta\in[-R_\zeta, R_\zeta]$.
    
    \item $\varrho_1(\zeta)$ is strictly monotonic on the interval $[-R_{\mathrm{mono}}, R_{\mathrm{mono}}]$ for some $R_{\mathrm{mono}}>0$.
\end{enumerate}
The second statement says that Criterion 2 is automatically satisfied when the chances of measuring either outcome are never certain. It also turns out that the invertible subspace requirement in Criterion 3 becomes trivially satisfied when there are only two measurement outcomes per detector, so we can place the monotonicity requirement directly on $\varrho_1(\zeta)$ instead. For details on why these are true, see the Corollaries below.

\subsection*{Convergence and Fading Memory}
\label{sec:proofs}

In this section, we will prove that stochastic reservoir computers are uniformly convergent and have the fading memory property at the level of the probability distributions. The strategy for proving this will utilize Theorem 3.1 of \cite{Grigoryeva:2018b} at its core. This theorem requires us to define a uniformly bounded family of sequences as well as a contracting reservoir map. A uniformly bounded family of sequences $K_R$ is defined to be the subset of sequences in $(\mathbb{R}^n)^{\mathbb{Z}_-}$ that obey
\begin{align}
    K_R = \left\{(u_k)_{k\in\mathbb{Z}_-}\in(\mathbb{R}^n)^{\mathbb{Z}_-}\right|\left.~||u_k||\leq R~\forall k\in\mathbb{Z}_-\right\},
\end{align}
where $R$ is a positive number. The vector norm $||\cdot||$ can be chosen arbitrarily, but for what follows it will be convenient to choose the 1-norm $||x||_1 = \sum_i |x_i|$. With $B_n[R]$ defined to be the closed ball of length $R$ with respect to the 1-norm in $\mathbb{R}^n$ centered on the origin, the reservoir map $f:B_N[R_x]\times B_n[R_u]\rightarrow B_N[R_x]$ is said to be contracting with respect to the 1-norm if for all  $v_1, v_2\in B_N[R_x],~u\in B_n[R_u]$ we have
\begin{align}
\label{eq:contract}
    ||f(v_1,u) - f(v_2,u)||_1 \leq \epsilon ||v_1 - v_2||_1
\end{align}
for some $0<\epsilon<1$.

Now we are ready to state a version of Theorem 3.1 from \cite{Grigoryeva:2018b} that will be of later use.
\begin{theorem}[Uniform Convergence and Fading Memory]
\label{thm:UCFM}
    With uniformly bounded sequences $K_{R_x} \subset (\mathbb{R}^N)^{\mathbb{Z}_-}$ and $K_{R_u} \subset (\mathbb{R}^n)^{\mathbb{Z}_-}$, if a continuous reservoir map $f:B_N[R_x]\times B_n[R_u]\rightarrow B_N[R_x]$ is contracting w.r.t the 1-norm, then the reservoir system associated with $f$ has the uniform convergence property, meaning that for any given input sequence $(u_k)_{k\in\mathbb{Z}_-}\in K_{R_u}$ there exists a unique solution $x\in B_N[R_x]$ to the equation $x_{k+1} = f(x_k, u_k),~\forall k\in\mathbb{Z}$. Furthermore, because $f$ is a continuous reservoir map, there exists a unique filter $U_f:K_{R_u}\rightarrow K_{R_x}$ associated with $f$ that is causal, time-invariant, and has the fading memory property, meaning that the filter $U_f$ is a continuous map between the metric spaces $(K_{R_u}, ||\cdot||_w)$ and $(K_{R_x}, ||\cdot||_w)$ using the weighted norm $||u||_w = \sup_{k\in\mathbb{Z}_-}\{||u_k w_{-k}||_1\}$, where $w:\mathbb{N}\rightarrow(0,1]$ is some decreasing sequence that asymptotically approaches zero.
\end{theorem}

To use Theorem \ref{thm:UCFM}, we must first establish that we have a continuous reservoir map between bounded spaces. With respect to the 1-norm, it is easy to see that every probability distribution $P_k$ belongs to $B_M[1]$ since all $M$ elements of $P_k$ are non-negative and must sum up to 1 for all $k\in\mathbb{Z}$. Eq. \eqref{eq:probres} maintains the 1-norm of the probabilities since the conditional probabilities obey $\sum_{a=1}^M p_{ab}(u) = 1$ for any $b$, so multiplication by the controlled transition matrix $p(u)$ does indeed qualify as a reservoir map $f:B_M[1]\times B_n[R_u]\rightarrow B_M[1]$. In fact, this maps vectors on the unit sphere $S_M[1]\in B_M[1]$ onto $S_M[1]$, with the unit sphere defined by $S_M(1) = \left\{x\in\mathbb{R}^M\left| ||x||=\sum_{i=1}^M |x_i|=1\right.\right\}$, so we can more specifically define $f:S_M[1]\times B_n[R_u]\rightarrow S_M[1]$ Assuming that $p(u)$ is continuous in $u\in B_n[R_u]$ for some radius $R_u$, the linearity of Eq. \eqref{eq:probres} w.r.t $P_k$ then guarantees that this reservoir map is continuous. We have now demonstrated that, given the assumptions about the continuity of $p(u)$ and that all $u_k$'s are bounded, the probabilities in a stochastic RC satisfy all criteria for Theorem \ref{thm:UCFM} aside from contraction.

Now we are ready to establish the contraction criteria for stochastic reservoir computing. The following theorem is essentially a statement about the strict contraction of finite-state Markov chains with respect to the 1-norm, which is a basic and fundamental result that is taught in courses on Markov chains \cite{Lalley:2016lec, Whitt:2013lec}. However, lectures notes like Refs. \cite{Lalley:2016lec, Whitt:2013lec} generally only provide sufficient conditions for contraction, while the conditions provided here are slightly stronger and are also necessary for convergence in the context of finite-state Markov chains w.r.t the 1-norm.
\begin{theorem}
\label{thm:contract}
    A stochastic reservoir map $f:S_M[1]\times B_n[R_u]\rightarrow S_M[1]$ defined by a controlled transition matrix $p(u)$ that is continuous as a function of $u$ is contracting w.r.t. the 1-norm if and only if the matrix $p^\top(u)p(u)$ has no zero elements for all $u\in B_n[R_u]$.
\end{theorem}
Since convergence criterion of the conditional probability matrix is a well-known result, the contents of Thm. \ref{thm:contract} are almost surely in previous literature, and we do not provide a proof in this section. For completeness, we do provide a proof in Supplementary Information \ref{sec:apndproof}.

With this theorem, we may now establish the full criteria for the uniform convergence and fading memory properties of the probability distributions in a stochastic RC.
\begin{lemma}
    With the uniformly bounded sequence $K_{R_u} \subset (\mathbb{R}^n)^{\mathbb{Z}_-}$, if a stochastic reservoir map $f:S_M(1)\times B_n[R_u]\rightarrow S_M(1)$ (using the 1-norm) defined by the controlled transition matrix $p(u)$ is continuous in $u$ and $p^\top(u)p(u)$ has strictly positive elements for all $u\in B_n[R_u]$, then the reservoir system associated with $f$ has the uniform convergence and fading memory properties with respect to the probabilities for obtaining each outcome.
\end{lemma} 
This is simply proven by combining Theorems \ref{thm:UCFM} and \ref{thm:contract}.

\subsection*{Polynomial Algebra}
\label{sec:polyalg}
Our next goal is to show that the class of stochastic ESNs that obey the uniform convergence property and the fading memory property with respect to a null sequence $w$ defined as part of Thm.~\ref{thm:UCFM} forms a polynomial algebra. More specifically, writing $\mathcal{X}=\{x^{(1)},\dots,x^{(M)}\}$, for a stochastic reservoir $F:\mathcal{X}\times B_n[R_u]\rightarrow \mathcal{X}$ with an associated controlled transition matrix $p(u)$ where multiplication by $p$ generates a map $p:S_M[1]\times B_n[R_u]\rightarrow S_M[1]$, we define a functional $G_{p,W}:K_{R_u}\rightarrow\mathbb{R}$ associated with $p(u)$ and the weight vector $W$. The recursive application of the probability map $p$ using an input sequence $(u_k)_{k\in\mathbb{Z}_-}\in K_{R_u}$ yields a unique sequence of probability distributions $(P_k)_{k\in\mathbb{Z}_-}\in(S_M[1])^{\mathbb{Z}_-}$. Multiplication of one of these probabilities by the weight vector $W$ generates a map $W:S_M[1]\rightarrow \mathbb{R}$. Thus the functional $G_{p,W}$ for a stochastic RC is the map from a given input sequence $(u_k)_{k\in\mathbb{Z}_-}\in K_{R_u}$ to some output $\hat{y}_0\in\mathbb{R}$, which is the inner product of $W$ with the probability distribution generated by these inputs and the controlled transition matrix $p(u)$ at some arbitrary time, taken to be $k=0$. In the event that the target sequence for our RC is vector-valued, we can view the output as a vector whose elements are functionals $G_{p,W_i}$, each using the same controlled transition matrix $p(u)$ but different weight vectors $W_i$. Thus without loss of generality we can focus on scalar outputs knowing that the results will straightforwardly generalize to vector outputs as well.

With the set $\mathcal{G}_w$ of these functionals, which are all continuous under the weighted norm $||u||_w = \sup_{k\in\mathbb{Z}_-}\{||u_k w_{-k}||_1\}$ for a null sequence $w$, we say that $\mathcal{G}_w$ forms a polynomial algebra if for any $G_{p_1,W_1}\left((u_k)_{k\in\mathbb{Z}_-}\right), G_{p_2,W_2}\left((u_k)_{k\in\mathbb{Z}_-}\right)\in \mathcal{G}_w$ we have that $G_{p_1,W_1}\left((u_k)_{k\in\mathbb{Z}_-}\right) + \lambda G_{p_2,W_2}\left((u_k)_{k\in\mathbb{Z}_-}\right)\in\mathcal{G}_w$ for any scalar $\lambda$ and $G_{p_1,W_1}\left((u_k)_{k\in\mathbb{Z}_-}\right) * G_{p_2,W_2}\left((u_k)_{k\in\mathbb{Z}_-}\right)\in\mathcal{G}_w$ using standard scalar addition and multiplication. We specifically want to focus on the set $\mathcal{G}_w^{\varrho}$ of functionals corresponding to stochastic ESNs with a specific controlled probability distribution $\varrho$, in which any member $G_{A,B,W}^{\varrho}\left((u_k)_{k\in\mathbb{Z}_-}\right)$ of the set corresponds to a controlled probability matrix $p(u)$ defined by $\varrho, A,$ and $B$ given in Eq. \eqref{eq:ESNprob}.

We begin by noting that if we take the matrix $A$ of any stochastic ESN of size $N$ to be block diagonal with blocks of size $N_1$ and $N_2 = N-N_1$, and also separate $B$ so that
\begin{align}
\label{eq:block}
    A=\left(\begin{matrix}
        A_1 & 0 \\ 0 & A_2
    \end{matrix}\right), \quad \quad B = \left(\begin{matrix}
        B_1 \\ B_2
    \end{matrix}\right)
\end{align}
The resulting outcomes $x^{(b)}$ will decompose analogously into independent vectors $x_1^{(b_1)}$ and $x_2^{(b_2)}$, with $x^{(b)\top} = \left(\begin{matrix} x_1^{(b_1)\top} & x_2^{(b_2)\top} \end{matrix}\right)$. These components separately obey their own reservoir equations, so that the controlled composite probability distribution and transition matrix factors such that
\begin{align}
    p_{s,a_s}(z) &= \prod_{i=1}^{N_s}\varrho_{a_{s,i}}(z_i) \\
    p_{ab}(u) &= p_{1,a}(A_1x_1^{(b_1)}+B_1u)p_{2,a}(A_2x_2^{(b_2)}+B_2u),
\end{align}
with $s\in\{1,2\}$ and where $a_{s,i}$ is defined so that $\xi^{(a_{s,i})} \equiv x_{s,i}^{(a_s)}$, meaning that, when equipped with weight vectors $W_1$ and $W_2$, these subsystems are themselves stochastic ESNs corresponding to functionals $G_{A_1,B_1,W_1}^{\varrho}\left((u_k)_{k\in\mathbb{Z}_-}\right)$ and $G_{A_2,B_2,W_2}^{\varrho}\left((u_k)_{k\in\mathbb{Z}_-}\right)$ belonging to the class $\mathcal{G}_w^{\varrho}$. This means that they generate their own probability distributions $P_{1,k}$ and $P_{2,k}$, and it is not difficult to see that the distributions $P_k$ for the original ESN using $A$ and $B$ are related by $P_{k,a}=P_{1,k,a_1}P_{2,k,a_2}$ for all $a_1\in\{1,\dots,m^{N_1}\}$ and $a_2\in\{1,\dots,m^{N_2}\}$, with the outcome $a$ corresponding to measuring outcomes $a_1$ and $a_2$ simultaneously.

Form here, it is easy to show that $\mathcal{G}_w^{\varrho}$ forms a polynomial algebra. Taking the product of the two functionals defined above yields $G_{A_1,B_1,W_1}^{\varrho}\left((u_k)_{k\in\mathbb{Z}_-}\right) * G_{A_2,B_2,W_2}^{\varrho}\left((u_k)_{k\in\mathbb{Z}_-}\right) = (W_1^\top P_{1,0}) (W_2^\top P_{2,0}) = (W_1\otimes W_2)^\top(P_{1,0}\otimes P_{2,0})$. However, $P_{1,0}\otimes P_{2,0}$ is just $P_0$ for the larger ESN that obeys Eq. \eqref{eq:block}, so with $W=W_1\otimes W_2$ it is clear that the product belongs to $\mathcal{G}_w^{\varrho}$. For the additive property, we must consider $G_{A_1,B_1,W_1}^{\varrho}\left((u_k)_{k\in\mathbb{Z}_-}\right) + \lambda G_{A_2,B_2,W_2}^{\varrho}\left((u_k)_{k\in\mathbb{Z}_-}\right) = W_1^\top P_{1,0} +\lambda W_2^\top P_{2,0}$. Since $P_{1,0}$ and $P_{2,0}$ are probability distributions, we can define the vector $e_N$ of length $N$ such that $e_{N,i} = 1$ for all $i\in\{1,\dots,N\}$ so that $e_{N_1}^\top P_{1,0} = e_{N_2}^\top P_{2,0} = 1$. We can then expand $W_1^\top P_{1,0} = (W_1\otimes e_{N_2})^\top (P_{1,0}\otimes P_{2,0})$ and $W_2^\top P_{2,0} = (e_{N_1}\otimes W_2)^\top (P_{1,0}\otimes P_{2,0})$ to get $W_1^\top P_{1,0} +\lambda W_2^\top P_{2,0} = (W_1\otimes e_{N_2} + \lambda e_{N_1}\otimes W_2)^\top (P_{1,0}\otimes P_{2,0})$. Therefore, using the the larger ESN that obeys Eq. \eqref{eq:block} again but with $W = W_1\otimes e_{N_2} + \lambda e_{N_1}\otimes W_2$, we see that the sum of the functionals is also in $\mathcal{G}_w^{\varrho}$. Hence $\mathcal{G}_w^{\varrho}$ forms a polynomial algebra for any controlled probability distribution function $\varrho$.

\subsection*{Separation}
The next and most involved step is to prove that the class $\mathcal{G}_w^{\varrho}$ of stochastic ESNs defined by a scalar probability distribution function $\varrho$ separates points in its domain $K_{R_u}$ of bounded input sequences. We say that $\mathcal{G}_w^{\varrho}$ has the separation property if for any $(u_k)_{k\in\mathbb{Z}_-}, (v_k)_{k\in\mathbb{Z}_-}\in K_{R_u}$ with $(u_k)_{k\in\mathbb{Z}_-}\neq (v_k)_{k\in\mathbb{Z}_-}$ we can always find at least one functional $G_{A,B,W}^{\varrho}\in\mathcal{G}_w^{\varrho}$ such that $G_{A,B,W}^{\varrho}\left((u_k)_{k\in\mathbb{Z}_-}\right)\neq G_{A,B,W}^{\varrho}\left((v_k)_{k\in\mathbb{Z}_-}\right)$. To prove this for a class $\mathcal{G}_w^{\varrho}$, we will look at the simplest case of single-node ESNs. This property will not hold for arbitrary probability distributions $\varrho$, so we will have to place conditions on $\varrho$ that are necessary for satisfying this property.

\begin{definition}
    The vector space $\mathcal{V}_R^p\subseteq\mathbb{R}^n$ for some natural number $n$ is called the {\bf invertible subspace} with respect to the controlled transition matrix $p(u)$ that is a function of $u\in[-R,R]$ if it is the largest vector space that satisfies $p(u) \mathcal{V}_R^p = \mathcal{V}_R^p$ for all $u\in[-R,R]$.
\end{definition}

\begin{theorem}[Separation]
\label{thm:separation}
    Let $\mathcal{G}_w^{\varrho}$ be the class of stochastic ESNs defined by a controlled probability distribution function $\varrho(\zeta)$ that is defined on an interval $\zeta\in[-R_\zeta, R_\zeta]$ for some $R_\zeta>0$. Also let $\mathcal{V}_{R_u}^p\subseteq\mathbb{R}^m$ be the invertible subspace with respect to the controlled transition matrix $p(u)$ whose elements are defined by $p_{ab}(u) = \varrho_a(A \xi^{(b)} + B u)$ for $m$ distinct measurement outcomes $\xi^{(b)}$ labeled by $b\in\{1,\dots,m\}$ and for $u\in[-R_u,R_u]$. The class $\mathcal{G}_w^{\varrho}$ has the separation property if for some $R_{\mathrm{mono}}>0$ there exists a vector $V\in\mathcal{V}_{R_u}^p$ for which $V^\top\varrho(\zeta)$ is strictly monotonic on $\zeta$ in the interval $[-R_{\mathrm{mono}},R_{\mathrm{mono}}]$.
\end{theorem}
\begin{proof}
    To prove that the class $\mathcal{G}_w^{\varrho}$ of stochastic ESNs separates points in $K_{R_u}$, we only need to show that for any pair of input sequences $(u_k)_{k\in\mathbb{Z}_-}$ and $(v_k)_{k\in\mathbb{Z}_-}\neq (u_k)_{k\in\mathbb{Z}_-}$ we can find at least one member $G_{A,B,W}^{\varrho}\left((u_k)_{k\in\mathbb{Z}_-}\right)$ of the class $\mathcal{G}_w^{\varrho}$ for which $G_{A,B,W}^{\varrho}\left((u_k)_{k\in\mathbb{Z}_-}\right)\neq G_{A,B,W}^{\varrho}\left((v_k)_{k\in\mathbb{Z}_-}\right)$. In what follows, it will turn out that there is always a single-node functional that separates any two points, so we will restrict our analysis to the case where $A$ and $B$ are both scalars. Consider a single-node stochastic ESN that generates a functional in $\mathcal{G}_w^{\varrho}$. The general probability of obtaining a given outcome $\xi^{(a)}$ is given by $P_{k,a}$. For a stochastic ESN, The evolution of these probabilities are given by 
    \begin{align}
        P_{k+1,a}=\sum_{b=1}^m\varrho_{a}(A \xi^{(b)} + B u_k)P_{k,b},
    \end{align}
    where $u_k, A,$ and $B$ are all scalars and there are $m$ possible outcomes. For vector inputs, $B$ would also become a vector, and the proof that follows can be applied by taking all but one of the components of $B$ to be zero. We write $P_k$ to refer to the vector whose components are the probabilities $P_{k,a}$. The functional is defined by $G_{A,B,W}^{\varrho}(u) = W^\top P_0$ for some $m$-dimensional vector $W$. 

    We are given two sequences $(u_k)_{k\in\mathbb{Z}_-}, (v_k)_{k\in\mathbb{Z}_-}\in K_{R_u}$ for which $(u_k)_{k\in\mathbb{Z}_-}\neq (v_k)_{k\in\mathbb{Z}_-}$, meaning that there is at least one $k\in\mathbb{Z}_-$ for which $u_k\neq v_k$. Let $k_\delta$ be the largest number in $\mathbb{Z}_-$ for which $u_{k_\delta}\neq v_{k_\delta}$, so that $u_k = v_k$ for all $k\in\{-1,-2,\dots,k_\delta+1\}$. Without loss of generality, we can take $u$ to be the sequence with the greater value of $u_{k_\delta}$, so that $u_{k_\delta}-v_{k_\delta}>0$. With $p_{ab}(u_k) = \varrho_{a}(A \xi^{(b)} + B u_k)$, we see that 
    \begin{align}
    G_{A,B,W}^{\varrho}\left((u_k)_{k\in\mathbb{Z}_-}\right) &= W^\top \left(\prod_{\kappa=1}^{-k_\delta-1}p(u_{-\kappa})\right)p(u_{k_\delta})P_{k_\delta}\left((u_k)_{k\in\mathbb{Z}_-}\right) \\
    &= \widetilde{W}^\top p(u_{k_\delta})P_{k_\delta}\left((u_k)_{k\in\mathbb{Z}_-}\right),    
    \end{align}
    where $\widetilde{W}^\top = W^\top \left(\prod_{\kappa=1}^{-k_\delta-1}p(u_{-\kappa})\right)$. In this expression we write the dependence of $P_{k_\delta}$ on the the input sequence $(u_k)_{k\in\mathbb{Z}_-}$ explicitly as $P_{k_\delta}\left((u_k)_{k\in\mathbb{Z}_-}\right)$, which is defined by
    \begin{align}
        P_{k_\delta}\left((u_k)_{k\in\mathbb{Z}_-}\right) = \mathop{\lim}_{ l \rightarrow \infty}  \left(\prod_{k=1}^{l} p(u_{k_\delta - k})\right) P_{k_\delta-l}\left((u_k)_{k\in\mathbb{Z}_-}\right),
    \end{align}
    where the limit is guaranteed to converge because of the contraction property of $p(u_k)$, so the dependence of $P_{k_\delta}\left((u_k)_{k\in\mathbb{Z}_-}\right)$ on $P_{k_\delta-l}\left((u_k)_{k\in\mathbb{Z}_-}\right)$ vanishes as $l\rightarrow\infty$. Because of how we defined $k_\delta$, we also have that $G_{A,B,W}^{\varrho}\left((v_k)_{k\in\mathbb{Z}_-}\right) = \widetilde{W}^\top p(v_{k_\delta})P_{k_\delta}\left((v_k)_{k\in\mathbb{Z}_-}\right)$ using the same vector $\widetilde{W}$. Thus by defining 
    \begin{align}
        \widetilde{G}_{A,B,W}^{\varrho}(u_k, P) \equiv W^\top p(u_k)P
    \end{align}
    as a function of the vector $W$, number $u_k$, and probability distribution $P$, we see that the filters can be written as 
    \begin{align}
    \label{eq:Gtilde}
        G_{A,B,W}^{\varrho}\left((u_k)_{k\in\mathbb{Z}_-}\right) = \widetilde{G}_{A,B,\widetilde{W}}^{\varrho}\left(u_{k_\delta}, P_{k_\delta}\left((u_k)_{k\in\mathbb{Z}_-}\right)\right),
    \end{align} with a similar expression for $v$.
    
    In general the controlled transition matrix $p$ may not be invertible, and so the resulting weight vector $\widetilde{W}$ may not be completely arbitrary. However, if we assume that $\widetilde{W}\in \mathcal{V}_{R_u}^p$ as defined in Thm. \ref{thm:separation}, then we can always find another vector $V\in\mathcal{V}_{R_u}^p$ such that $V^\top p(u_k) = \widetilde{W}^\top$ for any $u_k\in[-R_u,R_u]$. Since $V\in\mathcal{V}_{R_u}^p$ as well, we can repeat this process as many times as necessary to find the vector $W\in \mathcal{V}_{R_u}^p$ that satisfies $W^\top\left(\prod_{\kappa=1}^{-k_\delta-1}p(u_{-\kappa})\right) = \widetilde{W}^\top$ for any product $\prod_{\kappa=1}^{-k_\delta-1}p(u_{-\kappa})$. Thus for any sequence $u$ we can always find a vector $W$ that will generate an arbitrary vector $\widetilde{W}\in\mathcal{V}_{R_u}^p$.

    The dependence on the rest of the sequence through $P_{k_\delta}\left((u_k)_{k\in\mathbb{Z}_-}\right)$ in the expression for $G_{A,B,W}^{\varrho}\left((u_k)_{k\in\mathbb{Z}_-}\right)$ is very difficult to analyze, but we can take advantage of the fact that it is a probability distribution to eliminate the dependence on the rest of the sequence. Because of Eq. \eqref{eq:Gtilde}, it will suffice to show that $G_{A,B,W}^{\varrho}\left((u_k)_{k\in\mathbb{Z}_-}\right)\neq G_{A,B,W}^{\varrho}\left((v_k)_{k\in\mathbb{Z}_-}\right)$ by demonstrating that there is always at least one $\widetilde{W}, A$, and $B$ such that $\widetilde{G}_{A,B,\widetilde{W}}^{\varrho}(u_{k_\delta}, P_1) \neq \widetilde{G}_{A,B,\widetilde{W}}^{\varrho}(v_{k_\delta}, P_2)$ for all distributions $P_1$ and $P_2$. Taking $R(\zeta)=\widetilde{W}^\top \varrho(\zeta)$ for any scalar $\zeta\in[-R_\zeta,R_\zeta]$, we have 
    \begin{align}
        \widetilde{G}_{A,B,\widetilde{W}}^{\varrho}(u_k, P) = \widetilde{W}^\top p(u_k)P = \sum_{b=1}^m R(A\xi^{(b)} + B u_k)P_b.
    \end{align}
    Since $R(\zeta)$ is a scalar function, taking an average over a set of values $\mathcal{Z}=\{\zeta_1,\dots,\zeta_m\}$ for the function will give a value between the minimum and maximum values of $R(\zeta)$ for $\zeta\in\mathcal{Z}$. Thus we have that 
    \begin{align}
    \label{eq:interval}
        \widetilde{G}_{A,B,\widetilde{W}}^{\varrho}(u_k, P)\in[\min_\xi R(A \xi + B u_k), \max_\xi R(A\xi + B u_k)],
    \end{align}
    where $\min_\xi f(\xi) = \min\{f(\xi^{(1)}), \dots, f(\xi^{(m)})\}$ and likewise for $\max_\xi f(\xi)$. Thus we need to show that the above intervals corresponding to $u_{k_\delta}$ and $v_{k_\delta}$ do not overlap for some choice of $A, B,$ and $\widetilde{W}$. 

    Let us assume that there is some $\widetilde{W}\in\mathcal{V}_{R_u}^p$ such that $R(\zeta)$ is strictly monotonic in an interval $[-R_{\mathrm{mono}},R_{\mathrm{mono}}]$ for some $R_{\mathrm{mono}}>0$. Choose $A>0$ and $B>0$ such that $\zeta_{\mathrm{max}}=A|\xi|_{\mathrm{max}}+B R_u\leq R_{\mathrm{mono}}$, where $|\xi|_{\mathrm{max}} = \max\{|\xi^{(1)}|,\dots,|\xi^{(m)}|\}$. Because $R(\zeta)$ is strictly monotonic on this interval, it is also bijective, so any interval of the form in Eq. \eqref{eq:interval} corresponds to an interval $[A\xi_\mathrm{min} + B u_k, A\xi_\mathrm{max} + B u_k]\subset [-R_{\mathrm{mono}}, R_{\mathrm{mono}}]$. Since we took $u_{k_\delta} > v_{k_\delta}$, we only need to show that
    \begin{align}
        A\xi_\mathrm{min} + B u_{k_\delta} > A\xi_\mathrm{max} + B v_{k_\delta} \quad \implies \quad u_{k_\delta} - v_{k_\delta} > \frac{A}{B}(\xi_\mathrm{max} - \xi_\mathrm{min}),
    \end{align} which can always be achieved by choosing $A < B\frac{u_{k_\delta} - v_{k_\delta}}{\xi_\mathrm{max} - \xi_\mathrm{min}}$. Therefore, under the above monotonicity assumption, since we have only placed finite upper bounds on the positive numbers $A$ and $B$, and because $\widetilde{W}\in\mathcal{V}_{R_u}^p$, for any pair of sequences $(u_k)_{k\in\mathbb{Z}_-}, (v_k)_{k\in\mathbb{Z}_-}\in K_{R_u}$ we can always find choices of $A,B,$ and $W$ that result in a functional $G_{A,B,W}^{\varrho}\left((u_k)_{k\in\mathbb{Z}_-}\right)$ that separate $(u_k)_{k\in\mathbb{Z}_-}$ and $(v_k)_{k\in\mathbb{Z}_-}$. This means that the class of stochastic ESNs $\mathcal{G}_w^{\varrho}$ contains members that can distinguish between any two input sequences $(u_k)_{k\in\mathbb{Z}_-},(v_k)_{k\in\mathbb{Z}_-}\in K_{R_u}$ and therefore separates points in $K_{R_u}$.
\end{proof}

The case where there are $m=2$ outcomes only is the simplest case, which also makes it the most well-studied case in terms of potential hardware and the most practical case to physically implement. In particular, any qubit-based quantum computing architecture can be used to realize this class of stochastic ESNs. The conditions for separability also simplify in this case, so we find it useful to state the $m=2$ version of Theorem \ref{thm:separation} as a corollary.
\begin{corollary}
\label{thm:separation2}
    The class of stochastic ESNs $\mathcal{G}_w^{\varrho}$ defined by a controlled probability distribution $\varrho(\zeta) = \left(\begin{matrix} \varrho_1(\zeta) & (1-\varrho_1(\zeta)) \end{matrix}\right)^\top$ defined on the interval $\zeta\in[-R_\zeta,R_\zeta]$ for some $R_\zeta>0$ has the separation property if for some $R_{\mathrm{mono}}>0$ we have that $\varrho_1(\zeta)$ is strictly monotonic on the interval $[-R_{\mathrm{mono}},R_{\mathrm{mono}}]$.
\end{corollary}
\begin{proof}
    The controlled transition matrix $p(u)$ for this system has the form
    \begin{align}
        p(\zeta_1,\zeta_2) = \left(\begin{matrix}
            \varrho_1(\zeta_1) & \varrho_1(\zeta_2) \\
            1-\varrho_1(\zeta_1) & 1-\varrho_1(\zeta_2)
        \end{matrix}\right)
    \end{align}
    The determinant of this matrix is given by
    \begin{align}
        \det(p(\zeta_1,\zeta_2)) = \varrho_1(\zeta_1)(1-\varrho_1(\zeta_2)) - \varrho_1(\zeta_2)(1-\varrho_1(\zeta_1)) = \varrho_1(\zeta_1) - \varrho_1(\zeta_2)
    \end{align}
    With $\zeta_1, \zeta_2\in[-R_{\mathrm{mono}}, R_{\mathrm{mono}}]$ for the value of $R_{\mathrm{mono}}>0$ for which $\varrho_1(\zeta)$ is assumed to be strictly monotonic, the only way this determinant can be 0 is if $\zeta_1=\zeta_2$. Thus with $\zeta_{1,2}=A \xi^{(1,2)} +B u_k$ and $A>0$ for distinct measurement outcomes $\xi^{(1)}$ and $\xi^{(2)}$, there is no value of $u$ for which $\zeta_1=\zeta_2$, so as long as we choose $A$ and $B$ such that $\zeta_1, \zeta_2\in[-R_{\mathrm{mono}}, R_{\mathrm{mono}}]$ for all $u\in[-R_u,R_u]$, we can guarantee that $p(u)$ is invertible for all $u\in[-R_u,R_u]$. This means that the invertible subspace $\mathcal{V}_{R_u}^p$ is simply $\mathbb{R}^2$. So by choosing $V=\left(\begin{matrix} 1 & 0 \end{matrix}\right)^\top\in\mathbb{R}^2$, we have that $V^\top\varrho(\zeta)=\varrho_1(\zeta)$ is strictly monotonic on the interval $[-R_{\mathrm{mono}}, R_{\mathrm{mono}}]$, satisfying the separability criterion established in Thm. \ref{thm:separation}.
\end{proof}

\subsection*{Proof of Theorem \ref{thm:universal} }
With polynomial algebra proven and sufficient criteria for separability given for classes of stochastic ESNs, we are now finally ready to apply the Stone-Weierstrass Theorem \cite[Theorem 7.3.1]{Dieudonne:2011} and prove the universality of these classes.
\begin{proof}[Proof of Theorem \ref{thm:universal}]
    The conditions that the distribution $\varrho(\zeta)$ is continuous in $\zeta$ and satisfies $\varrho^\top(\zeta_1)\varrho(\zeta_2)>0$ for all $\zeta_1, \zeta_2\in[-R_\zeta, R_\zeta]$ ensures that the conditional probabilities for the ESNs are contracting by Theorem \ref{thm:contract}, so that the ESNs have the uniform convergence and fading memory properties by Theorem \ref{thm:UCFM}. To elaborate, the controlled composite probabilities of a general stochastic ESN with dimension $L$ are given in Eq. \eqref{eq:ESNprob} and have the form $p_a(z) = \prod_{i=1}^L\varrho_{a_i}(z_i)$. The continuity of $p_a(z)$ follows directly from the continuity of $\varrho(\zeta)$, and the condition for contraction is satisfied if $\sum_{a=1}^{m^L} p_a(z_1)p_a(z_2)>0$ for all $L$-dimensional vectors $z_1, z_2\in[-R_\zeta, R_\zeta]^L$. The sum over $a$ can be decomposed into $L$ sums over the $a_i$'s, so we have $\sum_a p_a(z_1)p_a(z_2)=\prod_{i=1}^L\left(\sum_{a_i=1}^m\varrho_{a_i}(z_{1,i})\varrho_{a_i}(z_{2,i})\right)=\prod_{i=1}^L\left(\varrho^\top(z_{1,i})\varrho(z_{2,i})\right)$. The positivity of this product immediately follows from the condition that $\varrho^\top(\zeta_1)\varrho(\zeta_2)>0$ for all scalars $\zeta_1, \zeta_2\in[-R_\zeta, R_\zeta]$, so every stochastic ESN that uses the probability distribution $\varrho$ has uniform convergence and fading memory. This is important because by Thm. \ref{thm:UCFM} it guarantees that each ESN corresponds to a unique, causal, and time-invariant filter $U_\varrho$ as well as a unique functional $G_{A,B,W}^{\varrho}\left((u_k)_{k\in\mathbb{Z}_-}\right)\in\mathcal{G}_w^{\varrho}$ on the space of uniformly bounded sequences $K_{R_u}$. Also, we can ensure that any $z_k=Ax^{(a)}+B u_k$ is in the interval $[-R_\zeta, R_\zeta]^L$ for any outcome $x^{(a)}$ and any input $u_k\in[-R_u,R_u]$ with appropriate choices of the matrix $A$ and vector $B$.

    To apply the Stone-Weierstrass Theorem on the class of functionals $\mathcal{G}_w^{\varrho}$, it must form a polynomial algebra, contain the constant functionals, and separate points in $K_{R_u}$. We established in a previous section that any $\mathcal{G}_w^{\varrho}$ forms a polynomial algebra with no additional conditions. It is also easy to show that by taking $W$ such that $W_a=C$ for all indices $a$, the functional $G_{A,B,W}^{\varrho}\left((u_k)_{k\in\mathbb{Z}_-}\right)=W^\top P_0=\sum_aC P_{0,a}=C$, so any class $\mathcal{G}_w^{\varrho}$ does indeed contain the constant functionals. Finally, by Theorem \ref{thm:separation} the additional condition that for some $R_{\mathrm{mono}}>0$ there exists a vector $V\in\mathcal{V}_{R_u}^p$ for which $V^\top\varrho(\zeta)$ is strictly monotonic on the interval $[-R_{\mathrm{mono}}, R_{\mathrm{mono}}]$ guarantees that $\mathcal{G}_w^{\varrho}$ separates points in $K_{R_u}$. These three properties allow us to apply the Stone-Weierstrass theorem to show that the class $\mathcal{G}_w^{\varrho}$ is dense on the space of functionals over $(K_{R_u}, ||\cdot||_w)$, and so one can always find a functional $G_{A,B,W}^{\varrho}\left((u_k)_{k\in\mathbb{Z}_-}\right)\in\mathcal{G}_w^{\varrho}$ that approximates some arbitrary functional $G^*\left((u_k)_{k\in\mathbb{Z}_-}\right)$ on $K_{R_u}$ to arbitrary precision, making the class $\mathcal{G}_w^{\varrho}$ of stochastic ESNs a universal approximating class.
\end{proof}

As noted in the previous section, it is particularly useful to consider the case when $m=2$. Here we state Theorem \ref{thm:universal} given this specific case as another corollary, which results in some simplification of the conditions necessary for universality.
\begin{corollary}
\label{thm:universal2}
    With the set of uniformly bounded sequences $K_{R_u} \subset (\mathbb{R}^n)^{\mathbb{Z}_-}$ and a weighted metric $||\cdot||_w$ defined in Thm. \ref{thm:UCFM}, let $\mathcal{G}_w^{\varrho}$ be the class of functionals generated by stochastic ESNs defined by a controlled probability distribution $\varrho(\zeta) = \left(\begin{matrix} \varrho_1(\zeta) & (1-\varrho_1(\zeta)) \end{matrix}\right)^\top$ that is defined on the interval $\zeta\in[-R_\zeta, R_\zeta]$, where $\varrho_1(\zeta)$ is continuous in $\zeta$ and either $\varrho_1(\zeta)\neq 0$ or $\varrho_1(\zeta)\neq 1$ for all $\zeta\in[-R_\zeta, R_\zeta]$. If for some $R_{\mathrm{mono}}>0$ we have that $\varrho_1(\zeta)$ is strictly monotonic on the interval $[-R_{\mathrm{mono}}, R_{\mathrm{mono}}]$, then the class $\mathcal{G}_w^{\varrho}$ is dense in the space of functionals over the compact metric space $(K_{R_u}, ||\cdot||_w)$.
\end{corollary}
\begin{proof}
    This is proven using exactly the same reasoning as in the proof of Theorem \ref{thm:universal}, but with the simpler condition for separation of points in Corollary \ref{thm:separation2} when the probability distribution $\varrho$ has only two components. Additionally, the contraction constraint that $\varrho^\top(\zeta_1)\varrho(\zeta_2)>0$ for all $\zeta_1, \zeta_2\in[-R_\zeta, R_\zeta]$ can be simplified by observing that the only way that this product can be zero for 2-dimensional probability distributions is if one of the vectors is $\left(\begin{matrix} 1 & 0 \end{matrix}\right)^\top$ and the other is $\left(\begin{matrix} 0 & 1 \end{matrix}\right)^\top$. Thus we can ensure that $\varrho^\top(\zeta_1)\varrho(\zeta_2)>0$ for all $\zeta_1, \zeta_2\in[-R_\zeta, R_\zeta]$ by requiring that $\varrho_1(\zeta)$ is either never equal to $0$ or never equal to $1$ on $[-R_\zeta, R_\zeta]$.
\end{proof}

Finally, we conclude this section by stating the universality of the class of all stochastic RCs as defined in Section \ref{sec:SRC}.
\begin{corollary}
    The class $\mathcal{G}_w$ of all stochastic RCs with a contracting controlled transition matrix $p(u)$ is universal.
\end{corollary}
\begin{proof}
    Because every stochastic RC in $\mathcal{G}_w$ has contracting transition matrix $p(u)$ it will have an associated functional in $\mathcal{G}_w$ that is unique by Thm. \ref{thm:UCFM}, so the class is well-defined in terms of the individual stochastic RCs. Any class $\mathcal{G}_w^{\varrho}$ of stochastic ESNs that satisfies the conditions for universality in Theorem \ref{thm:universal} is contained in $\mathcal{G}_w$, so we can always find a member $G_{A,B,W}^{\varrho}\left((u_k)_{k\in\mathbb{Z}_-}\right)\in\mathcal{G}_w^{\varrho}\subset\mathcal{G}_w$ that will approximate any causal, time-invariant filter to an arbitrary error tolerance. The fact that $G_{A,B,W}^{\varrho}\left((u_k)_{k\in\mathbb{Z}_-}\right)\in\mathcal{G}_w$ as well for all $G_{A,B,W}^{\varrho}\left((u_k)_{k\in\mathbb{Z}_-}\right)$ shows that $\mathcal{G}_w$ is also a universal approximating class, so in this way the universality of $\mathcal{G}_w$ is inherited from the universality of $\mathcal{G}_w^{\varrho}$.
\end{proof}

\section*{Data and Code Availability}
The data generated for the figures in this article as well as the Jupyter notebooks written in Python used to make them are available in a figshare repository \href{https://doi.org/10.6084/m9.figshare.27957141}{https://doi.org/10.6084/m9.figshare.27957141}.

\begin{appendix}
\setcounter{section}{18}

\section{Supplementary Information}

\subsection{Universality of Proposed Designs}
\label{sec:apnduniversal}

We can show that qubit reservoir networks are universal using Corollary 2. If we choose $\delta=\frac{\pi}{4}$ and restrict ourselves to choices of matrices $A$ and vectors $B$ of dimension $L$ for which  $z_k = A x_k + B u_k$ are in an interval $[-R_\zeta,R_\zeta]^L$ at every time step $k$ where $R_\zeta<\frac{\pi}{4}$, then we guarantee that $0<\sin^2(z_{k,i}+\frac{\pi}{4})<1$ for any $k$ and $i$ since $0<z_{k,i}+\frac{\pi}{4}<\frac{\pi}{2}$. Also, we note that $\sin^2(\zeta+\frac{\pi}{4})$ is strictly increasing for all $\zeta\in[-R_\zeta,R_\zeta]$ if $R_\zeta<\frac{\pi}{4}$. Thus the controlled probability $\varrho_1(\zeta)$ given in Eq. (19) for qubit reservoir networks with $\delta=\frac{\pi}{4}$ meets the requirements for a universal approximating class as long the linearly transformed outcomes $z_k=A x_k + B u_k$ have components that are always in the interval $[-R_\zeta,R_\zeta]^L$ for some $R_\zeta<\frac{\pi}{4}$.

We can also show that our stochastic optical networks form a universal approximating class. Comparing to the conditions for universality in Corollary 2, the controlled probability given in Eq. (20) satisfies that $\varrho_1(\zeta) \neq 1$ for any finite $\zeta$. However, because it is an even function of $\zeta$, it is only strictly monotonic on the intervals $(-\infty,0)$ and $(0,\infty)$. Similar to the previous example, this means we have to choose an offset $d>0$ and make sure that $A$ and $B$ are chosen so that every component of $z_k$ is in some closed interval $[-R_\zeta, R_\zeta]$ where $R_\zeta<d$. Thus by choosing some $d>0$ and $0<R_\zeta<d$ so that at every time step $k$ and vector component $i$ we have $z_{k,i}\in[-R_\zeta,R_\zeta]$, we obtain a universal approximating class using this stochastic ESN design.

\subsection{Noise Analysis}
\label{sec:apndnoise}

For a realistic implementation of a stochastic RC, we will need to get an estimate for the probabilities $P_{k,a}$ of obtaining an outcome labeled by $a$ at time step $k$ by averaging over many runs of the RC. Our procedure for doing this is as follows. Each run evolves the readouts $x_k$ according to Eq. (1), without any averaging or error mitigation. After running the RC for $N_{\mathrm{runs}}$ times, we then count the number of occurrences $n_k^{(a)}$ of the outcome $x_k = x^{(a)}$ at time step $k$ across all of the runs. For a stochastic reservoir with $M$ outcomes, we then estimate 
\begin{align}
    P_{k,a}\approx \frac{n_k^{(a)}+1}{N_{\mathrm{runs}}+M}.
\end{align}
The additional factors of $+1$ and $+M$ come from a Bayesian analysis of the problem assuming no prior information about the distribution, where the above expression for $P_{k,a}$ represents the expected value of $P_{k,a}$ given $N_{\mathrm{runs}}$ prior measurements of $x_k$ where $n_k^{(a)}$ of them resulted in $x^{(a)}$. This result can be found in general texts on Bayesian analysis such as Ref. \cite{Heckerman:2008}. The main benefit to using this estimator is that the estimates will never be exactly equal to 0, making it possible to divide by components of the probability vector without having to check for a null probability.

\begin{figure}
    \centering
    \includegraphics[width=0.45\linewidth]{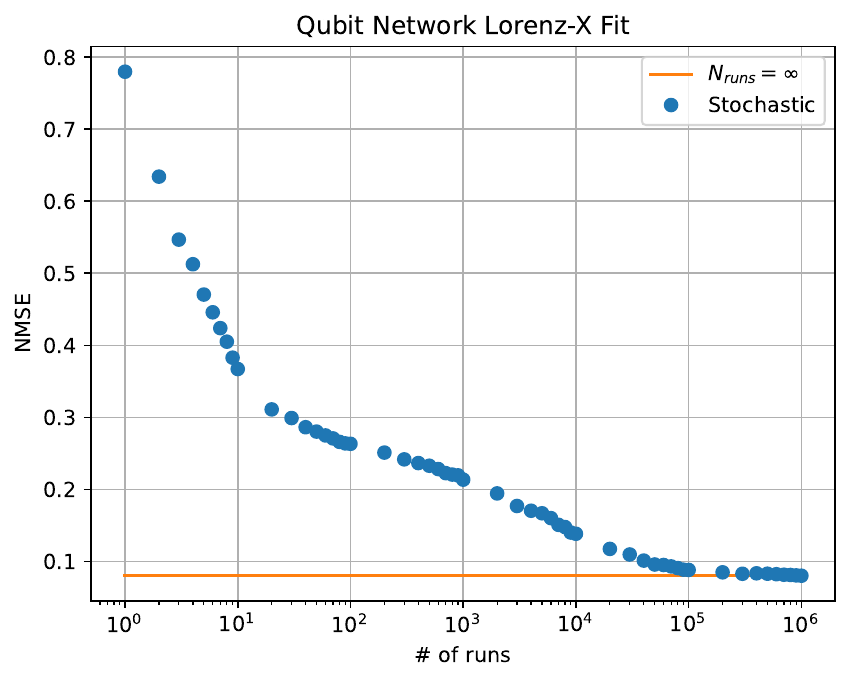}
    \includegraphics[width=0.45\linewidth]{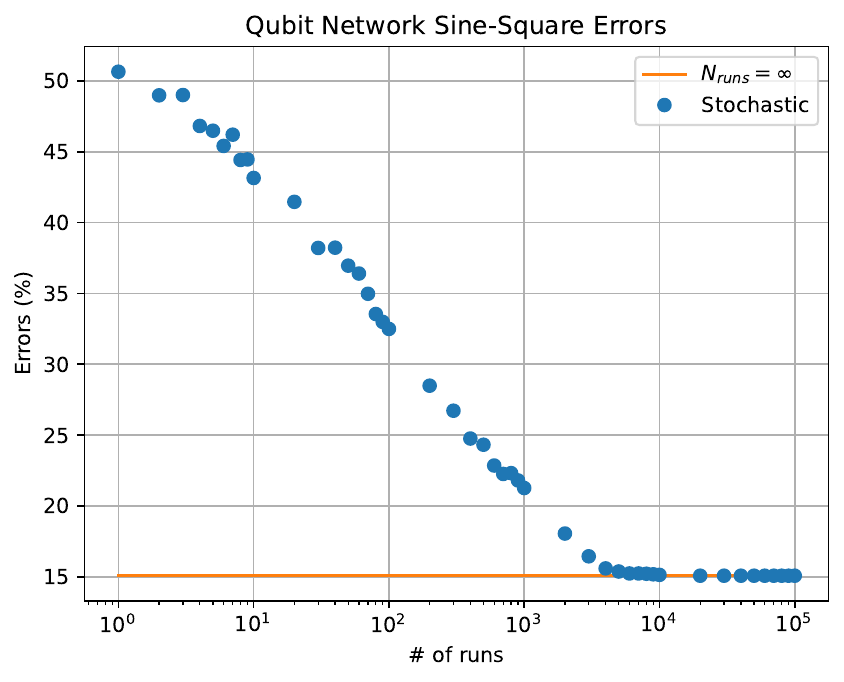}
    \caption{The left plot shows NMSE values for the Lorenz $X$ task, while the right plot shows the error percentage for the Sine-Square wave identification task, both as a function of the number of runs for the qubit reservoir network. Each data point represents the relevant error measure of the stochastic ESN with fixed $A$ and $B$ using 2 detectors.}
    \label{fig:SineSquareRuns}
\end{figure}

\begin{figure}
    \centering
    \includegraphics[width=0.45\linewidth]{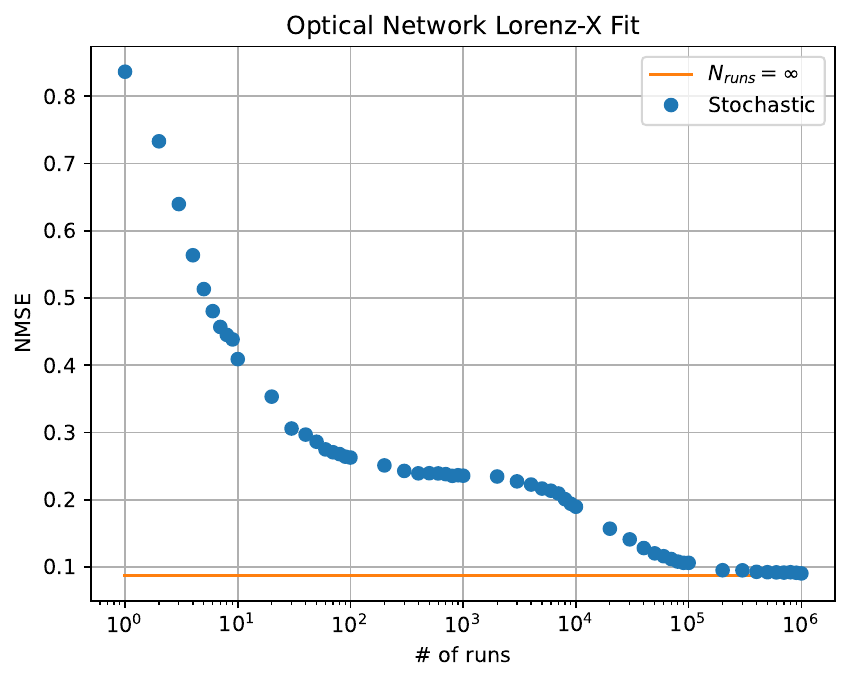}
    \includegraphics[width=0.45\linewidth]{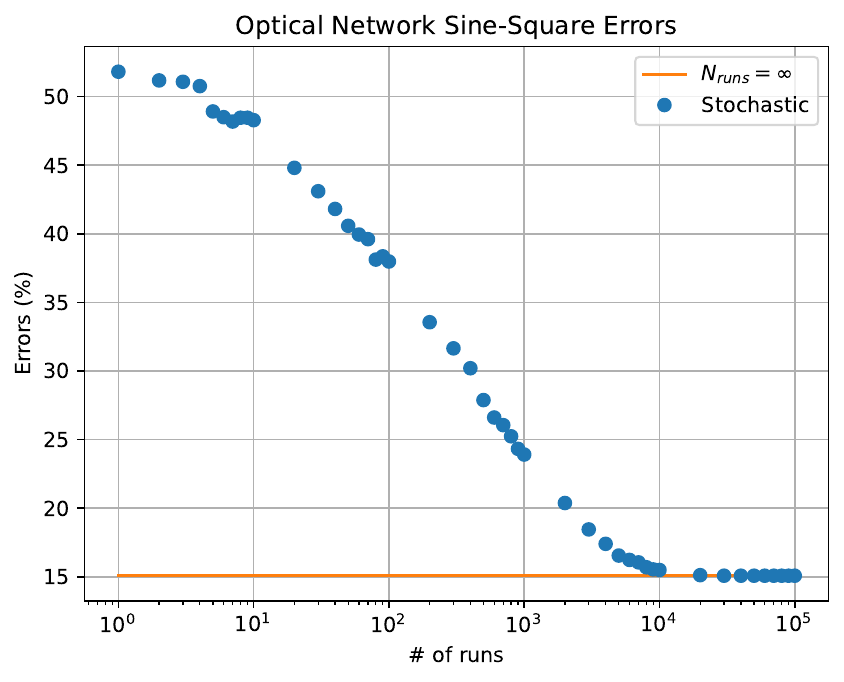}
    \caption{The left plot shows NMSE values for the Lorenz $X$ task, while the right plot shows the error percentage for the Sine-Square wave identification task, both as a function of the number of runs for the stochastic optical network. Each data point represents the relevant error measure of the stochastic ESN with fixed $A$ and $B$ using 2 detectors.}
    \label{fig:LorenzRuns}
\end{figure}

In our numerical analysis, we observed that when estimating the values of $P_{k,a}$ with this method, the performance of the stochastic RC seems to level off after a certain number of nodes. We provide an explanation for this phenomenon through an analysis of how the linear regression behaves under the noise introduced by the stochastic reservoir. Obtaining the optimal solution for the weight vector $W$ involves calculating the matrix $K=\frac{1}{N}\sum_k P_k P_k^\top$ and its inverse $K^{-1}$, where the sum over $k$ is taken over the $N$ training data steps. In general, $K$ has eigenvalues that span a wide range of scales, from some values that are order 1 to some that are many orders of magnitude less. With a finite number of shots, the matrix $K_{\mathrm{est}} = \frac{1}{N}\sum_k P_{k,\mathrm{est}} P_{k,\mathrm{est}}^\top$ using the estimated probabilities $P_{k,\mathrm{est}}\approx P_k$ has an expectation value given by
\begin{align}
\label{eq:Kexpect}
    \langle K_{\mathrm{est}} \rangle = \left(1-\frac{1}{N_{\mathrm{runs}}}\right)K + \frac{1}{N_{\mathrm{runs}}}P_k^{\mathrm{diag}},
\end{align}
where $(P_k^{\mathrm{diag}})_{ab} = \delta_{ab}P_{k,a}$ is the probability vector $P_k$ broadcast into a diagonal matrix. This expression is closely related to the eigentask analysis of the effects of noise in expressive capacity introduced in Ref. \cite{Hu:2023}, and for $N_{\mathrm{runs}}=1$ will lead to some of the results of Ref. \cite{Polloreno:2023pre}.

This performance hit is illustrated for a 2-detector example in Fig. \ref{fig:SineSquareRuns}, where we plot the error measures for both tasks using the qubit reservoir network as a function of $N_{\mathrm{runs}}$ for specific choices of $A$ and $B$. For the qubit reservoir network on the Lorenz $X$ task, the smallest eigenvalue of $K$ is $\lambda_{\min}=3.73*10^{-5}$, while for the Sine-Square wave tasks the smallest eigenvalue of $K$ is $\lambda_{\min}=5.24*10^{-4}$. We can then expect heuristically that, with $N_{\mathrm{runs}}=100000$ for the Lorenz $X$ task and $N_{\mathrm{runs}}=10000$ in the Sine-Square wave task, we should be able to resolve every eigenstate of $K$. This is exactly what we see in Fig. \ref{fig:SineSquareRuns}, where the error measures are only decreasing logarithmically with the number of runs until about $N_{\mathrm{runs}} =  8*10^4 \approx 3.0/\lambda_{\min}$ for the Lorenz $X$ task and $N_{\mathrm{runs}} = 4*10^3 \approx 2.1/\lambda_{\min}$, where they begin to converge to the error values derived using the exact probabilities.

We can perform the same analysis on the stochastic optical network for the 2-detector examples shown in Fig. \ref{fig:LorenzRuns}, where we plot the relevant error measures as a function of $N_{\mathrm{runs}}$ for specific choices of $A$ and $B$ (the same choices used in Fig. \ref{fig:SineSquareRuns}). For the Lorenz $X$ task, the smallest eigenvalue of $K$ is $\lambda_{\min}=1.32*10^{-5}$, while for the Sine-Square wave tasks the smallest eigenvalue of $K$ is $\lambda_{\min}=2.35*10^{-4}$. We can then expect heuristically that, with $N_{\mathrm{runs}}>100000$ for the Lorenz $X$ task and $N_{\mathrm{runs}}=10000$ in the Sine-Square wave task, we should be able to resolve every eigenstate of $K$. This is exactly what we see in Fig. \ref{fig:LorenzRuns}, where the error measures decrease logarithmically with the number of runs until roughly $N_{\mathrm{runs}} =  1.5*10^5 \approx 2.0/\lambda_{\min}$ for the Lorenz $X$ task and $N_{\mathrm{runs}} = 8*10^3 \approx 1.9/\lambda_{\min}$, where they begin to converge to the error values derived using the exact probabilities. Note that for the Lorenz $X$ task the trend appears to have a plateau from about $3*10^2$ to $2*10^3$ runs, which could have been taken as a signal that the theoretical minimum NMSE for this ESN is about $0.24$ in a realistic calculation where we have no prior knowledge of the true minimum. 

We can compare the performance boost in these examples due to increased statistics through Figs. \ref{fig:SineSquareRuns} and \ref{fig:LorenzRuns}. We see that in both cases the qubit reservoir network reaches the infinite statistics values a bit earlier than the stochastic optical network in both cases, which may explain the general performance difference we see in the other figures. Otherwise, the two plots are very similar, with the qubit reservoir perhaps showing a more steady improvement as a function of $N_{\mathrm{runs}}$ than the more curved trends for the optical network. Note that the values of $\lambda_{\min}$ we found for the Lorenz $X$ task are more than an order of magnitude smaller than for the Sine-Square wave task. This will influence where the stochastic ESN performance stalls out in Figs. 2 and 4, and it would explain why we needed $N_{\mathrm{runs}}=100000$ to get comparable results for the Lorenz $X$ task while only $N_{\mathrm{runs}}=10000$ was needed for the Sine-Square wave task.

\begin{figure}
    \centering
    \includegraphics[width=0.45\linewidth]{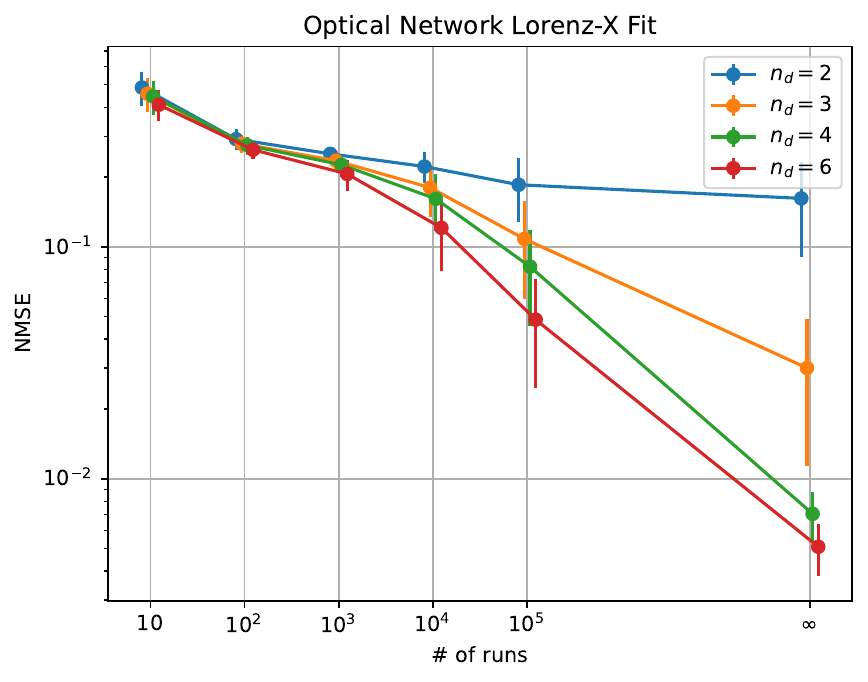}
    \includegraphics[width=0.45\linewidth]{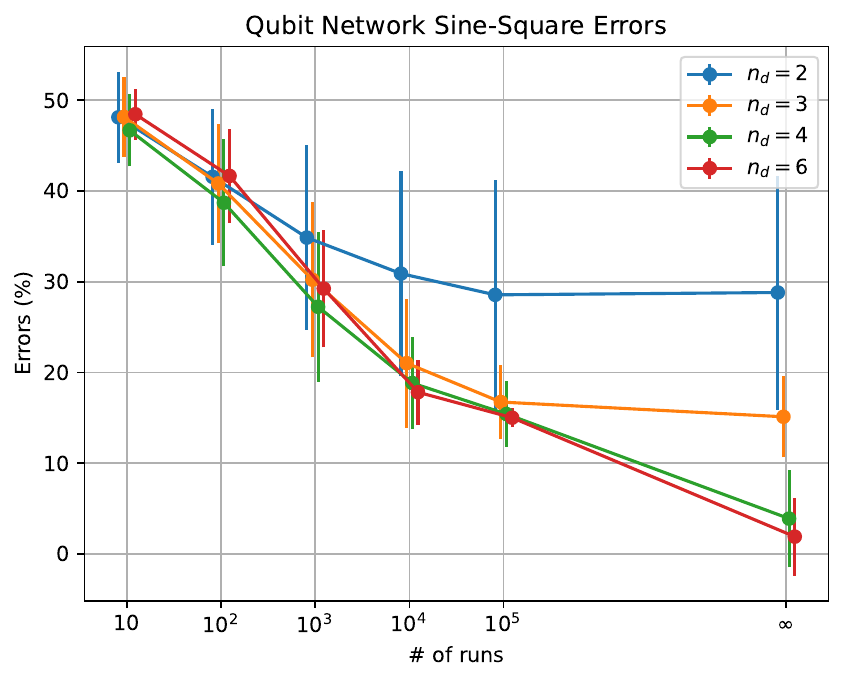}
    \caption{The left plot shows average NMSE values for the Lorenz $X$ task on the stochastic optical network, while the right plot shows the average error percentage for the Sine-Square wave identification task on the qubit reservoir network, both as a function of the number of runs of the stochastic ESN. The different lines shown correspond to the averages for ESNs with different numbers of detectors, specifically for 2, 3, 4, and 6 detectors. The point labeled $\infty$ on the x-axes of these plots correspond to the stochastic ESN results in the theoretical limit of infinite statistics, using the exact probabilities. Each data point represents the average over 100 random choices of $A$ and $B$, and the error bars give the standard deviations of each data point over the samples of ESNs with different choices of $A$ and $B$.}
    \label{fig:runs}
\end{figure}

\begin{figure}
    \centering
    \includegraphics[width=0.45\linewidth]{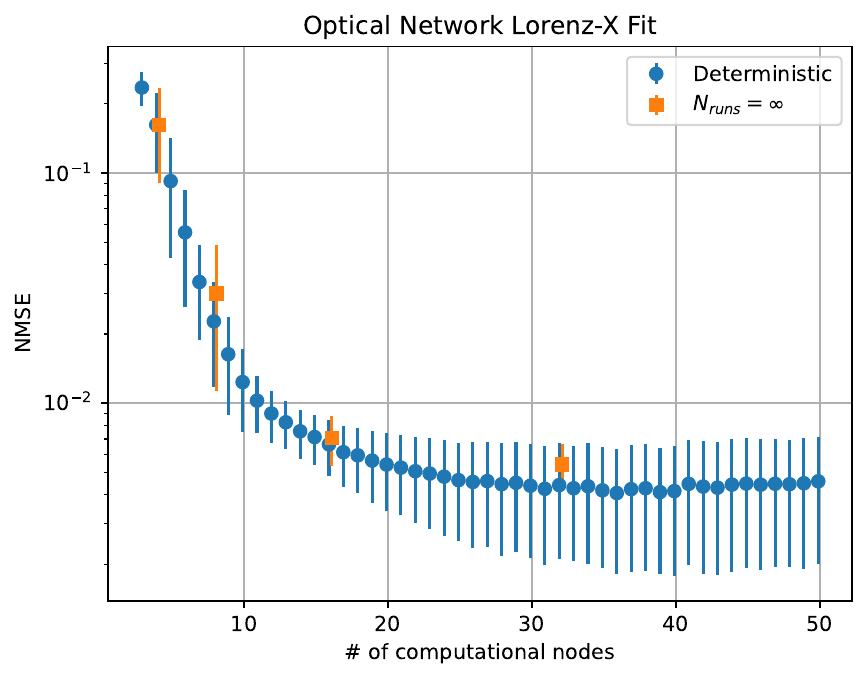}
    \includegraphics[width=0.45\linewidth]{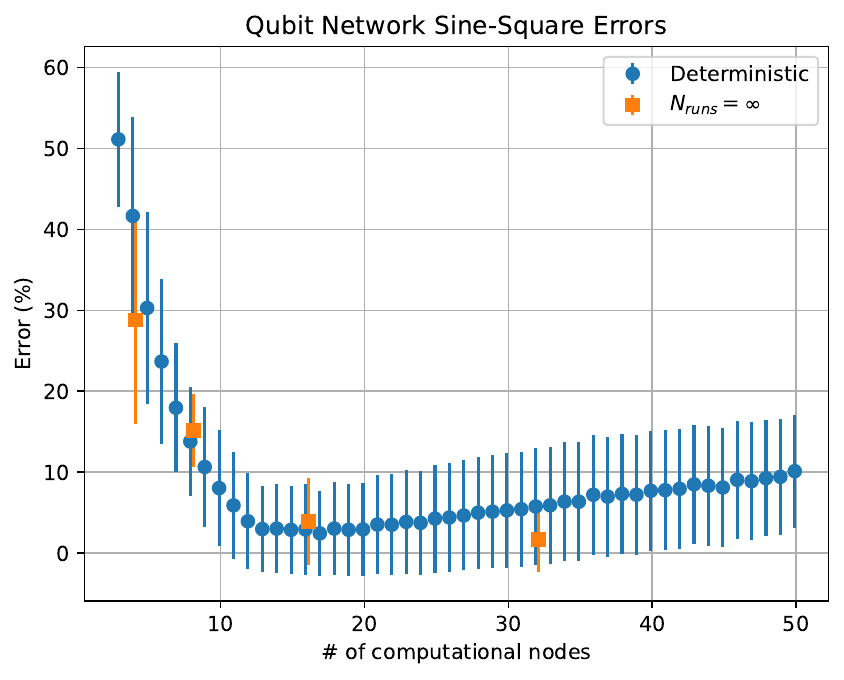}
    \caption{The left plot shows average NMSE values for the Lorenz $X$ task on the stochastic optical network, while the right plot shows the average error percentage for the Sine-Square wave identification task on the qubit reservoir network, both as a function of the number of computational nodes of the ESNs. The blue points show the NMSE values for the deterministic ESNs, while the orange points show the results for the stochastic ESN in the theoretical limit of infinite statistics, using the exact probabilities. Note that in this plot we are comparing the deterministic and stochastic reservoirs based on the number of computational nodes, rather than by the number of hardware detectors as done in previous figures. Each data point represents the average over 1000 random choices of $A$ and $B$, and the error bars give the standard deviations of each data point over the samples of ESNs with different choices of $A$ and $B$.}
    \label{fig:det}
\end{figure}

In Fig. \ref{fig:runs}, we show how the average performance of the stochastic ESNs approaches the limit of infinite statistics as the number of runs increases. This differs from Figs. \ref{fig:SineSquareRuns} and \ref{fig:LorenzRuns} in that it illustrates the average performance of the stochastic ESNs for $n_d\ge 2$ rather than the performance of specific ESNs with $n_d=2$. We see that the ESNs generally converge to the infinite run limit by 100000 runs when $n_d=2$, but for more detectors there remains a sizable gap in performance between these two points. This necessitates looking into noise mitigation techniques in future work to approach the theoretical limit more quickly for larger ESNs.

In the Discussion section, we observe a level off of the performance in the stochastic reservoirs as the number of detectors increase, and claim that this will also occur in deterministic RCs with more than 8 nodes. We suggest that this is due to the limited machine precision of our Python simulations, which induce errors that substantially limit the performance of larger RCs, even when the RC is otherwise deterministic.

In Fig. \ref{fig:det}, we compare the average performance of the deterministic and stochastic RCs in terms of the number of \textit{computational} nodes rather than the number of detectors as in previous figures. The contrast is in the fact that the previous figures were intended to illustrate the computational capabilities of deterministic and stochastic RCs as a function of their physical size, which is represented by the number of detectors. On the other hand, with Fig. \ref{fig:det}  we are now trying to present evidence of a purely computational phenomenon that arises as a result of training the weights associated with the RC, which seems to occur regardless of the physical characteristics of the RC, so a comparison based on the number of computational nodes is most appropriate here. For the deterministic RC, we have $n_c = n_d + 1$ because we have $\hat{y}_k=W^\top x_k + C$, with $n_d$ nodes coming from each component of $W$ and an extra node coming from $C$, while for stochastic RCs $n_c=2^{n_d}$ because $\hat{y}_k=W^\top P_k$. We emphasize again that for the same number of computational nodes $n_c$ a stochastic RC  can have a smaller physical size (smaller $n_d$) than its deterministic counterpart.

Fig. \ref{fig:det} shows the average performance of the deterministic RCs on the two tasks we focus on for up to 50 computational nodes, illustrating the level off claimed earlier. Also included are the data from the stochastic RC calculations in the limit of infinite statistics that were previously shown in Figs. 2 and 4. In the figure, we observe that the deterministic ESN in the Lorenz-X task exhibits diminishing returns after 10 computational nodes, and fails to show any improvement at around 30 nodes. For the number of errors in the Sine-Square task, we see that the best performance occurs at around 15 nodes in the deterministic reservoir, after which both the average and the variance increase with the number of nodes, presumably due to overfitting. This phenomenon is also evident in the stochastic reservoir using the exact probabilities, though for the Sine-Square wave task the best performance seems to occur closer to 32 computational nodes in that case.

Our explanation for why this occurs and how it relates to machine precision is the following.  In essence, the components of the readout vectors seem to be too similar to each other, so that the differences between them are orders of magnitude less than the vectors themselves. This means that the optimized training weights must amplify these small differences to get a different output function, which will lead to significant noise contamination if these differences are not resolved. In training, we must invert the matrix $K=\frac{1}{N}\sum_k P_k P_k^\top$ for the stochastic ESN, with $P_k$ replaced with $x_k$ for the deterministic ESNs. As the number of nodes increase in these tasks, we observe that the larger eigenvalues and their corresponding eigenvectors remain largely fixed as the number of nodes increases, while new eigenvectors are added which correspond to smaller and smaller eigenvalues. Eventually, these eigenvalues hit a lower limit based on machine precision where they tend to emerge at roughly the same order of magnitude rather than at lower magnitudes as before, indicating an fundamental change in the fit. We use Tikhonov regularization for the matrix inversion with a regularization parameter of $10^{-10}$, which provides the actual lower limit to these eigenvalues in our simulations and is known to produce an effect similar to noise contamination, but this regularization is also commonplace and necessary to avoid the direct consequences of machine precision errors in RCs with many computational nodes. Due to the wide-reaching implications of such an issue in reservoir computing, we believe that this deserves a more thorough investigation that is outside of the scope of this article.

\subsection{Simulation Details}
\label{sec:apnddetails}
\subsubsection{Task Details}
The Sine-Square wave identification task \cite{Dudas:2023} requires the ESN to identify whether a given section of the input is a sine wave or a square wave. The target sequence for this task is a random string $\{\tau(m)\}$ of ones and zeroes, with equal probability of choosing 1 or 0 for each $m$, where a zero corresponds to a sine wave input while a one corresponds to a square wave input. The period of both waves are tuned to be exactly 8 time steps, and at the end of each period the input wave has a random chance to switch to the other waveform. The formulas for theses waves in terms of $k$ are given by 
\begin{align}
    u_k &= \delta_{0,\tau\left(\mathrm{floor}\left(\frac{k}{8}\right)\right)}u_k^{(sin)} + \delta_{1,\tau\left(\mathrm{floor}\left(\frac{k}{8}\right)\right)}u_k^{(sqr)}, \\
    u_k^{(sin)} &= \sin\left(\frac{\pi}{4}k\right), \\
    u_k^{(sqr)} &= (-1)^{\mathrm{floor}\left(\frac{k}{4}\right)},
\end{align}
where $\mathrm{floor}\left(s\right)$ is the largest integer $n_s$ such that $n_s\leq s$ for any real number $s$. The task for the ESN is to distinguish the sine and square wave parts of the input, yielding an output of $\hat{y}_k<0.5$ for a sine wave and $\hat{y}_k>0.5$ for a square wave. In this work, we generated one random sine-square wave input and used it for all of the numerical work found in the following sections.

The other task we use in this work is the Lorenz $X$ task, which requires the ESN to approximate a chaotic dynamical system described by $X(t)$ in the Lorenz or equations:
\begin{align}
    \frac{dX}{dt}(t) &= a(Y-X) \\
    \frac{dY}{dt}(t) &= X(b-Z)-Y\\
    \frac{dZ}{dt}(t) &= XY-cZ
\end{align}
where we choose the standard values $a = 10, b = 28,$ and $c = 8/3$. We numerically approximate the solution to this equation using the fourth-order Runge-Kutta method with $\delta t = 1.0$ and initial conditions of $X(0) = 0.5, Y(0) = 0.1,$ and $Z(0) = 0.2$. We also run the solution for $1000$ steps before using it for the task, or in other words the target sequence we use actually starts with $y_{1000}$. The task for the ESN is to predict what the sequence will be 1 time step into the future. In other words, using the input sequence $(u_k)_{k\in\mathbb{Z}} = (y_{k-1})_{k\in\mathbb{Z}}$, we want the ESN to successfully predict $(y_k)_{k\in\mathbb{Z}}$.

\subsubsection{Simulation Parameters}
In the numerical analysis illustrated in Figs. 2 and 4, the left plot of both figures shows a comparison of NMSE values for the Lorenz $X$ task averaged over 100 different choices of $A$ matrices and $B$ vectors for the ESN, while the right plot shows a comparison of the percentage of errors recorded for the Sine-Square wave identification task averaged over 1000 choices. For the deterministic ESNs, we use the exact same choices of $A$ and $B$, using the probability function $\varrho_1(\zeta)$ as the deterministic activation function, so that the reservoir equations are given by
\begin{align}
    (x_{\mathrm{det}})_{k+1} &= \langle\Gamma(A (x_{\mathrm{det}})_k + B u_k)\rangle = \varrho_1(A (x_{\mathrm{det}})_k + B u_k) \\
    \hat{y}_k &= W (x_{\mathrm{det}})_k + C,
\end{align}
where the scalar activation function $\Gamma(\zeta)$ and controlled probability $\varrho_1(\zeta)$ are applied elementwise to their vector arguments. In Fig. 2, we use $\varrho_1(\zeta)$ given in Eq. (19) corresponding to the qubit reservoir network, while in Fig. 4 we use $\varrho_1(\zeta)$ given in Eq. (20) for the stochastic optical network. For the stochastic ESNs, we plot the cases where we average over a finite number of runs as well as the theoretical limit of infinite runs using the exact probabilities. We used 100000 runs in the finite case for the Lorenz $X$ task and 10000 runs for the Sine-Square wave task.

For each run, we used 3000 training data steps performed after 100 steps of initial startup time for the Lorenz $X$ task and 96 steps for the Sine-Square wave task. The values for the error measures shown here are for a set of time steps immediately after the training data set, totaling 500 steps for the Lorenz $X$ task and 504 for the Sine-Square wave task. These same parameters apply to the data in all of the figures from the previous section as well. However, for the Sine-Square wave plots in Figs. \ref{fig:SineSquareRuns} and \ref{fig:LorenzRuns}, each data point is an average over 10 repetitions of the error percentage calculation, using the same number of runs each time, to reduce the spread of the values in this plot.

\subsection{Proof of Theorem 3}
\label{sec:apndproof}
\begin{proof}
    Starting from the left-hand side of Eq. (22), we have for the stochastic reservoir map defined by $p(u)$ acting on two probability vectors $P_1, P_2\in S_M[1]$ and an input $u\in B_n[R_u]$
    \begin{align}
        ||f(P_1,u) - f(P_2,u)||_1 = ||p(u)(P_1 - P_2)||_1
    \end{align}
    Using the triangle inequality, we can see that 
    \begin{align}
    \label{eq:triproof}
        ||p(u)(P_1-P_2)||_1 &= \sum_{a=1}^M\left|\sum_{b=1}^Mp_{ab}(u)(P_1-P_2)_b\right| \\
        &\leq \sum_{a=1}^M\sum_{b=1}^M\left|p_{ab}(u)(P_1-P_2)_b\right| \\
        &= \sum_{b=1}^M\left(\sum_{a=1}^M p_{ab}(u)\right)\left|(P_1-P_2)_b\right| \\ 
        &= \sum_{b=1}^M(1)\left|(P_1-P_2)_b\right| = ||P_1-P_2||_1,
    \end{align}
    where in the third line we used the non-negativity of probabilities. 
    
    For the absolute value of two scalars $c$ and $d$, the inequality $|c+d|\leq |c|+|d|$ is saturated if and only if $c$ and $d$ are both non-negative or both non-positive, or more succinctly if the product $cd$ is non-negative. Thus for $n$ scalars the inequality $|\sum_{i=1}^nc_i|\leq \sum_{i=1}^n|c_i|$ is saturated if and only if all of the $c_i$'s are either non-negative or non-positive, or alternatively if the products $c_{i_1}c_{i_2}$ are non-negative for all $i_1, i_2\in\{1,\dots,n\}$. This means that for our use of the triangle inequality in Eq. \eqref{eq:triproof}, equality occurs if and only if for every $a\in{1,\dots,M}$ the products  $p_{ab_1}(u)p_{ab_2}(u)(P_1-P_2)_{b_1}(P_1-P_2)_{b_2}\geq 0$ for all $b_1, b_2\in\{1,\dots,M\}$. Therefore, the inequality is strict if and only if for some $a\in\{1,\dots,M\}$ there is a pair of distinct indices $b_1, b_2\in\{1,\dots,M\}$ for which $p_{ab_1}(u)p_{ab_2}(u)(P_1-P_2)_{b_1}(P_1-P_2)_{b_2}<0$. The inequality in Eq. \eqref{eq:triproof} will be strict if the triangle inequality is strict for just one index $a$ because the sum over $a$ in Eq. \eqref{eq:triproof} is over non-negative quantities, so if any one of the terms in the sum fails to saturate the triangle inequality, the entire sum will not saturate it, either. 
    
    For any two distinct probability distributions $||P_1-P_2||_1=\sum_{a=1}^M|P_{1,a}-P_{2,a}|>0$ but $\sum_{a=1}^M(P_{1,a}-P_{2,a}) = 1-1=0$. This implies that there is at least one element of $P_1-P_2$ that is positive and at least one other that is negative, so there will always be at least one pair of distinct indices $b_1, b_2$ for which $(P_1-P_2)_{b_1}(P_1-P_2)_{b_2}<0$. However, if we consider the distributions $P_{1,a} = \delta_{ab_1}$ and $P_{2,a} = \delta_{ab_2}$ with $b_1\neq b_2$, then we see that $b_1, b_2$ is the only pair of distinct indices for which $(P_1-P_2)_{b_1}(P_1-P_2)_{b_2}<0$. Thus in order for $p_{ab_1}(u)p_{ab_2}(u)(P_1-P_2)_{b_1}(P_1-P_2)_{b_2}<0$ to hold for any distinct probability distributions $P_1, P_2\in S_M(1)$, we must have that $p_{ab_1}(u)p_{ab_2}(u)\neq0$ for all pairs of distinct indices $b_1, b_2\in\{1,\dots,M\}$. Since $p_{ab}(u)$ is always non-negative, these inequalities are all satisfied when $p_{ab}(u)>0$ for all $b\in\{1,\dots,M\}$. Therefore, the inequality is strict for all distinct distributions $P_1, P_2\in S_M(1)$ if and only if for all pairs of distinct indices $b_1, b_2\in\{1,\dots,M\}$ we have for some $a\in\{1,\dots,M\}$ that $p_{ab_1}(u)p_{ab_2}(u)>0$.

    Finally, we note that because the controlled transition matrix $p(u)$ has non-negative entries for all inputs, $\sum_{a=1}^Mp_{ab_1}(u)p_{ab_2}(u)>0$ if and only if $p_{ab_1}(u)p_{ab_2}(u)>0$ for at least one $a\in\{1,\dots,M\}$. But $\sum_{a=1}^Mp_{ab_1}(u)p_{ab_2}(u) = \left(p^\top(u)p(u)\right)_{b_1b_2}$, so we can restate the condition for strict inequality using this matrix. Thus we have that the triangle inequality is strict for all distinct probability distributions $P_1, P_2\in S_M(1)$ and all inputs $u\in B_n[R_u]$ if and only if the matrix $p^\top(u)p(u)$ has strictly positive elements for all $u\in B_n[R_u]$. Also, because each element of $p(u)$ is continuous over all $u$ in the compact set $B_n[R_u]$, the image of $p(u)$ is also compact. This means that $d=\inf_u p^\top(u)p(u)$ is contained in the image of $p^\top(u)p(u)$ on $u\in B_n[R_u]$, and since $p^\top(u)p(u)$ has strictly positive elements for all $u\in B_n[R_u]$ we have that $d>0$, and so $p^\top(u)p(u)\geq d$ for some $d>0$. Therefore the contraction criterion
    \begin{align}
        ||p(u)(P_1-P_2)||_1 \leq \epsilon||P_1-P_2||_1
    \end{align}
    holds for a continuous probability matrix $p(u)$ and some $0<\epsilon<1$ if and only if the matrix $p^\top(u)p(u)$ has strictly positive elements for all inputs $u\in B_n[R_u]$ since $||p(u)(P_1-P_2)||_1 < ||P_1-P_2||_1$ for all distinct probability distributions $P_1, P_2\in S_M(1)$ and $||p(u)(P_1-P_2)||_1 = ||P_1-P_2||_1 = 0$ when $P_1=P_2$.
\end{proof}

In addition, we leave off with a remark that is intended to show that a fully stochastic reservoir map is not difficult to achieve. This is equivalent to the final, stronger convergence criterion given in Ref. \cite{Whitt:2013lec}.
\begin{remark}
    If for every $u\in B_n[R_u]$ there is at least one value of $a$ for which $p_{ab}(u)\neq 0$ for all $b\in\{1,\dots,M\}$, then $p^\top(u)p(u)$ has strictly positive elements for all inputs $u\in B_n[R_u]$.
\end{remark}
Since $\left(p^\top(u)p(u)\right)_{b_1b_2} = \sum_{a=1}^Mp_{ab_1}(u)p_{ab_2}(u)$, if there is at least one index $a$ for which $p_{ab}(u)\neq 0$ for all $b\in\{1,\dots,M\}$, then $p_{ab_1}(u)p_{ab_2}(u)>0$ for all $b_1, b_2\in\{1,\dots,M\}$ as well. If this holds for every input $u$, then this guarantees that $p^\top(u)p(u)$ always has strictly positive elements for all $u\in B_n[R_u]$. Thus if we can guarantee that for every $u$ there is at least one row of $p(u)$ for which every probability is nonzero, then it will be contracting.

\subsection{Note on a Recent Result}
\label{sec:apndrefnote}
Stochastic RCs as we define them have been considered in a recent paper by Polloreno \cite{Polloreno:2023pre}. In this work, it is shown that the information processing capacity of a single run of a stochastic RC only scales polynomially with the scale of the hardware. However, in our work we run the RC many times and get estimates for the probabilities $P_{k,a}$ by counting the number of occurrences of each outcome. The derivation of Theorem 2 in Ref. \cite{Polloreno:2023pre} is predicated on the fact that the second moment of the outcome probabilities $P_{k,a}$ is linear in these probabilities, but this is no longer true when $P_{k,a}$ is estimated with an average over multiple shots. In Supplementary Information \ref{sec:apndnoise}, we show that past a certain number of aggregated runs, the estimated probabilities provide a very good estimate of the results using the exact probabilities. Of course, it could still be argued that the number of runs needed to achieve this level of accuracy may be exponentially scaling with the scale of the hardware as well. In any case, we do not claim that stochastic RCs have exponentially scaling computational complexity. We merely state that the number of computational nodes of the reservoir is exponentially scaling, and these nodes may or may not be used effectively for any given computation.

\end{appendix}

\bibliography{SRC}

\end{document}